\documentclass[twoside]{article}

\usepackage[accepted]{aistats2022}
\usepackage{natbib}

\usepackage{hyperref}       
\usepackage{url}            
\usepackage{amsfonts}       

\usepackage{xcolor}
\usepackage{amsmath,amsthm}
\usepackage{algorithm}
\usepackage[noend]{algorithmic}

\newtheorem{definition}{Definition}[section]

\newtheorem{theorem}{Theorem}[section]

\newtheorem{lemma}[theorem]{Lemma}

\newcommand{\Ec}{\mathcal{E}}
\newcommand{\Fc}{\mathcal{F}}

\newcommand{\Ic}{\mathcal{I}}

\newcommand{\Nc}{\mathcal{N}}

\newcommand{\Sc}{\mathcal{S}}

\newcommand{\Pb}{\mathbb{P}}

\newcommand{\Eb}{\mathbb{E}}
\newcommand{\Rb}{\mathbb{R}}

\newcommand{\argmax}{\arg\max}

\newcommand{\Ent}{\mathrm{Ent}}

\allowdisplaybreaks

\begin{document}

%

%

\twocolumn[

\aistatstitle{Approximate Top-$m$ Arm Identification with Heterogeneous Reward Variances}

\aistatsauthor{ Ruida Zhou \And Chao Tian }

\aistatsaddress{ Texas A\&M University \And Texas A\&M University } ]

\begin{abstract}
  We study the effect of reward variance heterogeneity in the approximate top-$m$ arm identification setting. In this setting, the reward for the $i$-th arm follows a $\sigma^2_i$-sub-Gaussian distribution, and the agent needs to incorporate this knowledge to minimize the expected number of arm pulls to identify $m$ arms with the largest means within error $\epsilon$ out of the $n$ arms, with probability at least $1-\delta$. We show that the worst-case sample complexity of this problem is 
{\tiny
\begin{align}
\Theta\left( \sum_{i =1}^n \frac{\sigma_i^2}{\epsilon^2} \ln\frac{1}{\delta} + \sum_{i \in G^{m}} \frac{\sigma_i^2}{\epsilon^2} \ln(m) + \sum_{j \in G^{l}} \frac{\sigma_j^2}{\epsilon^2} \text{Ent}(\sigma^2_{G^{r}}) \right),\label{eqn:SC}
\end{align}}
where $G^{m}, G^{l}, G^{r}$ are certain specific subsets of the overall arm set $\{1, 2, \ldots, n\}$, and $\text{Ent}(\cdot)$ is an entropy-like function which measures the heterogeneity of the variance proxies. The upper bound of the complexity is obtained using a divide-and-conquer style algorithm, while the matching lower bound relies on the study of a dual formulation.
\end{abstract}

\section{Introduction}

In the multi-armed bandit (MAB) model, an agent interacts with a slot machine by pulling one of the many arms and observing the corresponding reward at each time step \citep{lattimore2020bandit}. The goal in the canonical MAB setting is to maximize the cumulative reward. In order to accomplish this goal, algorithms must be designed to balance exploration and exploitation during this online learning process; the objective in this setting is usually referred to as regret minimization. In many applications, the true goal is in fact not to maximize the cumulative reward, but to identify the best arm among all the arms, and regret minimization does not match the true goal in such cases. Instead, the best arm identification problem is the more suitable formulation, and it is a pure exploration problem \citep{bubeck2009pure} that aims to identify the best arm as \emph{fast} and \emph{accurately} as possible. 

Approximate best arm identification with fixed confidence is a formal PAC-learning formulation for the best arm identification setting, where the agent is required to identify an arm whose expected reward is not less than that of the best arm by $\epsilon$, with a confidence at least $1 - \delta$. A more general version of this problem is to identify the top-$m$ arms, where the expected rewards of the $m$ arms identified are not less than that of the $m$-th best arm by $\epsilon$, with a confidence at least $1-\delta$. In this setting, the algorithms will have a performance guarantee in terms of the confidence of success. We refer to these settings as $(\epsilon,\delta)$ best arm identification, and $(\epsilon,\delta)$ top-$m$ arm identification, respectively.  

In most previous works on multi-armed bandit, an inherent assumption is that the reward distribution of each arm is sub-Gaussian, and moreover, the variance proxies are homogeneous among all the arms. Such an assumption may be natural when the rewards are bounded in a range, or it is reasonable to view the arms as of the same randomness nature (except the mean rewards of the arms). In other applications, this assumption is less suitable, since the reward distributions are naturally heterogeneous. In this work, we consider $(\epsilon,\delta)$ best $m$-arm identification with sub-Gaussian distributed rewards when the variance proxies are heterogeneous and known. Our goal is to understand the worst-case sample complexity of the problem as a function of the number of arms $n$, the number of best arms to be identified $m$, the variance proxy vector $(\sigma^2_1,\sigma^2_2,\ldots,\sigma^2_n)$, i.e., the worst in the class of possible reward distributions with the given variance proxy vector; see  Section \ref{sec:pre} for a more precise definition.

For a more concrete example application of the problem setting, consider a remote sensing setting, where multiple underground sensors will need to communicate with the central controller on a wireless link to find the best location to drill for natural gas. The channel noises in the wireless link can be viewed as additive noises on the sensing values themselves, and such channel statistics are usually obtained independent of the sensing but by sending and receiving pilot signals. Therefore, the variances of the arms are indeed known, and the goal is to identify several ``best'' arms. Although the problem is well motivated by practical applications, our approach to study it is largely theoretical, and the obtained result is theoretical in nature. Particularly, it is well-known that the worst-case sample complexity in the homogeneous $(\epsilon, \delta)$ top-$m$ arm identification setting is $\Theta\left(\sum_{i \in [n]}\frac{\sigma^2_{i}}{\epsilon^2}(\ln\frac{1}{\delta} +\ln m) \right)$ (c.f., \citep{kalyanakrishnan2012pac}). However, we observe that the structure of the sample complexity may evolve (specifically, the factor $\ln(m)$ will diminish) as the setting transitions from the homogeneous setting to the heterogeneous setting. Therefore, we focus on this transition behavior and aim to provide a precise characterization \emph{theoretically }.

Several well-known algorithms can be straightforwardly adapted to the problem under consideration. We first consider adapting the naive elimination approach and the median elimination algorithm \citep{even2002pac, kalyanakrishnan2010efficient}, as well the LUCB \citep{kalyanakrishnan2012pac} and UGapE \citep{gabillon2012best} algorithms. We observe that the adapted algorithms only perform well in some respective cases. More precisely, the adapted naive elimination algorithm performs well when the heterogeneity is more significant, and the adapted median elimination algorithm performs well when the heterogeneity is less significant. Given this observation, we seek for a new algorithm that can naturally account for the heterogeneity, and propose the variance-grouped median elimination algorithm. There is no need to artificially ascribe an instance as having either high or low heterogeneity in this algorithm, and its performance adapts naturally.

We further establish a matching lower bound by reformulating it into an optimization problem, and considering its dual. Combined with this lower bound, we show the proposed algorithm is in fact optimal. The worst-case sample complexity, as given in (\ref{eqn:SC}), is in general proportional to the sum of the reward variances, and has three components. The first component (with $\ln \frac{1}{\delta}$) reflects the effect of the confidence parameter, the second component reflects the impact of the more homogeneous subset of the arms, and the last term (with the $\Ent(\cdot)$ function) reflects the impact of the more heterogeneous subset of the arms. The result naturally degrades if the reward variances are indeed homogeneous, which essentially has only the first two components. The third component captures the impact of the heterogeneity, which is not critically related to $m$, but on the variances $\sigma^2_{1:n}$ through an entropy-like function. For highly heterogeneous variances, the second term will in fact disappear, and $\textnormal{Ent}(\sigma^2_{G^{r}})$ can be of order $O(1)$, thus becoming independent of $m$ completely. 

\section{Related Works}

Multi-armed bandit problems have been extensively studied in the machine learning community in the past decades. A canonical setting is to maximize the cumulative reward, whose asymptotically optimal behavior was first characterized in the seminal work by \citet{lai1985asymptotically}. Good tutorials and books \citep{bubeck2012regret,slivkins2019introduction,lattimore2020bandit} are readily available.

An alternative setting is to instead identify the best arm. There are in general two lines of research: minimizing the mis-identification probability within a fixed budget of samples \citep{audibert2010best,bubeck2013multiple,carpentier2016tight}, and fast identification with a fixed confidence guarantee \citep{jamieson2014best}. The $(\epsilon, \delta)$ best arm identification problem belongs to the latter and was introduced in \citep{even2002pac, even2006action}, where several elimination based algorithms, such as naive elimination, successive elimination and median elimination algorithms, were proposed. The median elimination algorithm was shown to be worst-case optimal for which a matching lower bound was derived by \citet{mannor2004sample}. The asymptotic (large number of arms) optimal elimination algorithm was recently discovered \citep{hassidim2020optimal}, which was inspired by the idea of identifying the ``good arms'' \citep{katz2020true}. The case of exact best arm identification, i.e., $\epsilon = 0$, motivated algorithms that adapt to the underlining model and usually performs well in an instance-dependent manner  \citep{karnin2013almost, jamieson2014lil, chen2015optimal, garivier2016optimal, kaufmann2016complexity}.

There are multiple variants of the problem \citep{zhou2014optimal,shen2019universal,jin2019efficient,assadi2020exploration,chaudhuri2019pac}. One of the most natural generalization of the best arm identification problem is to identify multiple best arms. The $(\epsilon, \delta)$ top-$m$ arm identification was studied in \citep{kalyanakrishnan2010efficient}, in which an algorithm named ``halving" was proposed, and it bears similarity to the median elimination algorithm. It was later shown that the halving algorithm is indeed worst-case optimal \citep{kalyanakrishnan2012pac}. Though more adaptive algorithms were proposed later, such as LUCB \citep{kalyanakrishnan2010efficient} and UGapE \citep{gabillon2012best,kalyanakrishnan2012pac}, they are not worst-case optimal. For the case of exact top-$m$ arm identification, efforts toward understanding the instance-dependent sample complexity were also made \citep{kaufmann2013information,chen2017nearly,simchowitz2017simulator}.

Gaussian rewards with heterogeneous variances was considered in the earliest work on best arm identification  \citep{bechhofer1954single} in the fixed confidence setting, though without a theoretical analysis on the stopping time. The possible variance heterogeneity among arms gained attentions recently in the fixed budget setting \citep{faella2020rapidly}, where the confidence bounds are designed based on central limit theorem. Identifying the best arms in multiple bandits with possible heterogeneous variances was studied in the fixed budget setting \citep{gabillon2011multi}, where an elimination based algorithm was proposed to take variances into designing confidence bound. In the addition to the fixed budget setting, most recently \citet{lu2021variance} studied the best arm identification with unknown heterogeneous variances in the fixed confidence setting. They assumed the support of reward distribution is bounded, and proposed an elimination-based algorithm by first estimating the variances (with known upper bound on the variances) then utilizing the estimated variances in identifying the unique best arm based on Bernstein-style confidence bounds. The algorithm achieves near-optimal instance dependent performance. In comparison, we aim to study the \emph{worst-case sample complexity} with known variance proxies as inputs (the support of reward distribution may be unbounded), in the \emph{top-$m$ identification} problem. We propose an optimal algorithm with an \emph{exact matching} lower bound, and studied the impact of variances transition from the homogeneous setting to the heterogeneous setting in terms of the parameter $m$.

\section{Preliminary} \label{sec:pre}

\textbf{System model:} We largely follow the canonical sub-Gaussian bandit model, except the additional component related to the reward variances. 
A bandit instance $I$ is represented by a set of arm indices $[n] := \{1, 2, \ldots, n\}$ and the tuple of reward distributions $(\nu_{1}, \nu_{2}, \ldots, \nu_n)$. For any $i \in [n]$, pulling the $i$-th arm returns a reward observation, which is independently sampled from distribution $\nu_{i}$, where $\nu_{i}$ is a sub-Gaussian distribution with mean $\mu_i$ and variance proxy $\sigma^2_i$\footnote{A random variable $X$ follows some $\sigma^2$-sub-Gaussian distribution, if $\ln \Eb[e^{\lambda(X - \Eb[X])}] \leq \frac{\sigma^2 \lambda^2}{2}, ~\forall \lambda \in \Rb$, and $\sigma^2$ is called the variance proxy.}.
An arm is $\epsilon$-approximate top-$m$ if the mean reward of that arm is at least $\max^m_{i\in [n]} \mu_i - \epsilon$, where $\max^m_{i \in [n]}$ indicates the $m$-th largest (mean reward) value among the arms in $[n]$. With the knowledge of variance proxy values $\sigma^2_{1:n}$, but without the knowledge of mean values $\mu_{1:n}$, the agent actively learns the parameters of the sub-Gaussian bandit instance $I$ by observing independent reward samples. When there is no ambiguity from the context, we omit ``proxy" and simply refer to $\sigma^2_{1:n}$ as the reward variances.

\textbf{$(\epsilon, \delta)$ top-$m$ arm identification:} In the $(\epsilon, \delta)$ top-$m$ arm identification problem, the agent is required to identify some subset $R \subset [n]$ with $|R| = m$, such that, with probability at least $1 - \delta$, any arm in $R$ is $\epsilon$-approximate top-$m$. 

\textbf{Algorithm class:} Taking the parameters $(\epsilon, \delta, m, [n], \sigma_{1:n}^2)$ as input, an algorithm \textsf{A} deployed by the agent is represented by a tuple $(\pi_t, \rho_t)_{t \geq 1}$. During the learning process, the function $\pi_t$ selects an arm in $[n]$ based on the inputs of the algorithm as well as the previous observations before time step $t$  (i.e., the arms that were pulled). The function $\rho_t$ decides whether to stop based on the inputs of the algorithm as well as the available observations (the current observation and the previous observations before time step $t$). If $\rho_t$ decides to stop, it returns a set of arms $R^{\textsf{A}} \subset [n]$; otherwise, the process continues. Let $T^{\textsf{A}}$ be the time that the process stops, which is the number of samples observed by algorithm $\textsf{A}$. We only study the \emph{valid} algorithms that solve the $(\epsilon, \delta)$ top-$m$ arm identification when dealing with any bandit instance.

\textbf{Worst-case sample complexity:} The number of samples observed by the algorithm $T^{\textsf{A}}$ is a stopping time, whose expectation the agent aims to minimize. We study the \emph{worst-case sample complexity} for $(\epsilon, \delta)$ top-$m$ arm identification, which is an intrinsic quantity that measures the difficulty of the problem, and thus independent of the algorithm and $\mu_{1:n}$. Formally, the worst-case sample complexity of the $(\epsilon, \delta)$ top-$m$ arm identification problem under algorithm inputs $(\epsilon, \delta, m, [n], \sigma_{1:n}^2)$ is 
\begin{align}
\textnormal{SC}(\epsilon, \delta, m, [n], \sigma_{1:n}^2) := \inf_{\textsf{A}} \sup_{I \in \Ic(\sigma^2_{1:n})} \Eb_{I}[ T^{\textsf{A}} ],
\end{align}
where the infimum is taken over all valid algorithms, the supremum is taken over the instance class $\Ic(\sigma^2_{1:n})$ containing all the distribution tuples $\nu_{1:n}$ with variances $\sigma^2_{1:n}$, and the subscript $I$ in the expectation $\Eb_{I}[\cdot]$ indicates that it is with respect to the bandit model $I$. 

\textbf{Measure of heterogeneity: } For any positive vector $a_{1:n}$, define the entropy function as $\Ent(a_{1:n}) := -\sum_{j = 1}^n \hat{a}_i\ln \hat{a}_i$ with $\hat{a}_i=\frac{a_j}{\sum_{i = 1}^n a_i}$. It measures the heterogeneity of the vector $a_{1:n}$, and takes value within $(0, \ln(n)]$. Note that the entropy function is usually defined on the probability simplex, and we had slightly abused the notation by defining it for a positive vector. In this paper, we study the worst-case sample complexity, which is gap-independent.

\section{Main Result: Worst-case Sample Complexity} \label{sec:main}

The main result of this work is the characterization of the worst-case sample complexity $\textnormal{SC}(\epsilon, \delta, m, [n], \sigma_{1:n}^2)$. To present this result, we first introduce some additional notation. Let $\underline{\sigma} := \min_{i \in [n]} \sigma_i$, and partition $[n]$ into $k$ disjoint subsets $G_1, \ldots, G_k$, such that for any $j \in [k]$,
\begin{align}
    G_j := \{i \in [n]: 2^{j-1} \leq \sigma_i^2/\underline{\sigma}^2 < 2^j \}. 
\end{align}
Define two disjoint sets
\begin{align}
G^{m} := \cup_{j: |G_j| > 2m} G_j, \quad G^{l} := \cup_{j: |G_j| \leq 2m} G_j,
\end{align}
where $|\cdot|$ denotes the cardinality of the set. 
For each $j$ with $G_j \subset G^{l}$, let $G'_j = G_j$; for each $j$ with $G_j \subset G^{m}$, select $G'_j \subset G_j$ with $|G'_j| = 2m$, and denote $G^{r} := \cup_{j \geq 1} G'_j$ as a subset of the arms, such that $\Ent\left(\sigma^2_{G^{r}} \right)$ is maximized. (The superscripts of $G^{m}, G^{l}, G^{r}$ indicate ``more", ``less" and ``reduced", respectively).

The worst-case sample complexity $\textnormal{SC}(\epsilon, \delta, m, [n], \sigma_{1:n}^2)$ is summarized in the following theorem.
\begin{theorem} \label{thm:main}
Suppose $n > 2m$, $\epsilon > 0$ and $0 < \delta < 0.1$, then the worst-case sample complexity is 
\begin{align}
& \textnormal{SC}(\epsilon, \delta, m, [n], \sigma_{1:n}^2) = \notag \\
& \Theta\left( \sum_{i \in [n]} \frac{\sigma_i^2}{\epsilon^2} \ln\frac{1}{\delta} + \sum_{i \in G^{m}} \frac{\sigma_i^2}{\epsilon^2} \ln(m) + \sum_{j \in G^{l}} \frac{\sigma_j^2}{\epsilon^2} \Ent(\sigma^2_{G^{r}}) \right).
\end{align}
\end{theorem}

The following lemma upper bounds the entropy $\Ent(\sigma^2_{G^r})$ in the third component. 
\begin{lemma}\label{lem:red-up}
For any $m \geq 2$, $ \Ent(\sigma^2_{G^{r}}) \leq 8 \ln(m)$.
\end{lemma}
This lemma indicates that the worst-case sample complexity in the heterogeneous variance setting is upper bounded by $O\left(\sum_{i \in [n]}\frac{\sigma^2_i }{\epsilon^2} \ln\frac{m}{\delta} \right) $ in general. In certain sense, the heterogeneity in fact makes the problem ``easier" to solve. To further illustrate this point, let us consider two special cases:
\begin{itemize}
\item When the variances are more homogeneous, e.g., in the extreme case $\sigma^2_i = \sigma^2,~ \forall i \in [n]$, we have $G^{m} = [n]$ and $G^{l} = \emptyset$. Theorem \ref{thm:main} naturally degrades to the worst-case sample complexity in the homogeneous setting characterized in \citep{kalyanakrishnan2012pac}, which is $\Theta\left(\frac{n \sigma^2}{\epsilon^2} \ln\frac{m}{\delta} \right)$.

\item When the variances are highly heterogeneous, e.g., in the extreme case $|G_j| = 1, \forall j = 1,2,\ldots, k$, we have $G^{m} = \emptyset$ and $G^{l} = [n]$. Theorem \ref{thm:main} shows that the worst-case sample complexity is $\Theta\left(\sum_{i \in [n]} \frac{\sigma_i^2}{\epsilon^2} \ln\frac{1}{\delta} \right)$, which is independent of $m$.
\end{itemize}
Comparing the two cases and assuming the sum of the variances remain the same, the latter clearly has a more desirable sample complexity. The sets $G^m$ and $G^l$ describe the transition between the homogeneous and the heterogeneous. In the rest of this article, we present the optimal algorithm and the matching lower bound to establish Theorem \ref{thm:main}.

\section{Algorithms} \label{sec:alg}

We first revisit several existing algorithms designed mostly under the assumption of homogeneous variances. By adapting them to the heterogeneous variance case, we analyze their advantages and disadvantages. As will become clear shortly, these adapted algorithms still perform well in certain respective cases. Based on this observation, we will propose an optimal divide-and-conquer style algorithm. 

\subsection{Adapting Existing Algorithms}

\paragraph{Weighted naive elimination:} 
In this adapted algorithm, the agent simply pulls each arm-$i$ a total of $\frac{2 \sigma_i^2}{(\epsilon/2)^2}\ln\frac{1}{\omega_i}$ times, calculates the sample mean $\hat{\mu}_i$, and returns the $m$ arms with the largest sample means. We call it ``weighted'' because the numbers of pulls for the arms are determined by the reward variances $\sigma^2_{1:n}$ and the confidence parameters $\omega_{1:n}$. The parameters $\omega_{1:n}$ need to be optimized in order to provide the performance guarantee, and the following lemma provides one such assignment of the optimized $\omega_{1:n}$. 

\begin{lemma} \label{lem:union}
Let $\omega_i = \delta \frac{\sigma_i^2}{\sum_{j=1}^n \sigma_j^2}$, the weighted naive elimination algorithm takes
\begin{align}
8 \sum_{i \in [n]}\frac{\sigma_i^2}{\epsilon^2} \left( \ln\frac{1}{\delta} + \Ent(\sigma_{1:n}^2)\right) \label{eqn:lem-union}
\end{align}
samples, and solves the $(\epsilon, \delta)$ top-$m$ arm identification problem for any $\epsilon > 0$ and $0 < \delta < 1$. 
\end{lemma}

We will use $\textsf{WNElim}(\epsilon, \delta, m, [n], \sigma_{1:n}^2)$ to denote the weighted naive elimination algorithm with the choices of $\omega_{1:n}$ in Lemma \ref{lem:union}. The entropy function $\Ent(\sigma^2_{1:n})$ appears naturally as a multiplicative factor in the second item of Equation (\ref{eqn:lem-union}), which measures the heterogeneity of the variances. If the variance heterogeneity is high, the entropy term $\Ent(\sigma^2_{1:n} )$ can be significantly less than $\log n$. As mentioned earlier, when $\sigma^2_i = 2^i$, the entropy term is in fact $O(1)$, i.e., no longer a function of $n$ and $m$. On the other hand, by the principal of maximum entropy \citep{cover1999elements}, it has the maximum value $\ln(n)$ when the variances are homogeneous. Thus the weighted naive elimination algorithm will provide good performance when the arm variances are highly heterogeneous, but will lose efficiency when they are more homogeneous.

\paragraph{Adapted median elimination:} Median Elimination (``Halving" algorithm in \citep{kalyanakrishnan2010efficient}) is known to achieve the worst-case optimal performance in the homogeneous variance setting. One simple method to adapt it to the heterogeneous setting is to ignore the knowledge of the heterogeneity, and simply assume that all the arms have the largest variance $\max_{i \in [n]}\sigma^2_i$. The original median elimination algorithm can be applied without any change, and the expected number of samples taken is thus $O\left(\frac{n \max_{i \in [n]} \sigma^2_i}{\epsilon^2}\left( \ln\frac{1}{\delta} + \ln m \right) \right)$, as shown in \citep{kalyanakrishnan2010efficient}. In the appendix, we provide another method to adapt the median elimination algorithm, which improves the $n \max_{i \in [n]} \sigma^2_i$ term by roughly halving the sum of variances in each round.

If the variances are more homogeneous, e.g., $\sigma_i^2 / \sigma_j^2 \leq 2, \forall i, j \in [n]$, then $\sum_{i \in [n]} \sigma^2_i \leq n \max_{i \in [n]} \sigma^2_i \leq 2 \sum_{i \in [n]} \sigma^2_i$ and the expected number of samples is $O\left(\frac{\sum_{i \in [n]} \sigma^2_i}{\epsilon^2}\left( \ln\frac{1}{\delta} + \ln m \right) \right)$. For the same example, the weighted naive elimination uses $O\left(\frac{\sum_{i \in [n]} \sigma^2_i}{\epsilon^2}\left( \ln\frac{1}{\delta} + \ln n \right) \right)$ samples. Thus this simple adaptation of the median elimination algorithm is able to perform well for the highly homogeneous case, but will induce a loss of performance for the more heterogeneous cases. 

\paragraph{Adapting other algorithms:} 
The adaptation of several instance dependent algorithms, such as LUCB and UGapE, is straightforward.  For the problem in consideration, both algorithms require $O\left( \frac{\sum_{i \in [n]} \sigma_i^2}{\epsilon^2}\left( \ln\frac{1}{\delta} + \ln \frac{\sum_{i \in [n]} \sigma_i^2}{\epsilon^2} \right) \right)$ number of samples in expectation in the worst case. They are not worst-case optimal in the homogeneous variance setting, and certainly not in the heterogeneous variance setting since the latter is a more general setting.

\subsection{The Optimal Variance-Grouped MedElim Algorithm}

It was shown in the previous subsection that the weighted naive elimination algorithm and the median elimination algorithm have advantages in the  respective cases. In order to retain the advantages in both algorithms, we take a ``divide and conquer" approach. Recall the minimum variance is $\underline{\sigma} = \min_{i \in [n]} \sigma_i$, and the disjoint subsets $G_1, \ldots, G_k$ form a partition of $[n]$, and for any $j \in [k]$,
\begin{align}
    G_j = \left\{i \in [n]: 2^{j-1} \leq \sigma_i^2/\underline{\sigma}^2 < 2^j \right\}. \label{def:G_j}
\end{align}
The largest variance ratio within each subset is at most $2$, while the variances among subsets are well separated. We wish to apply median elimination to each subset and select ``good" arms within that subset, and then apply weighted naive elimination over all the selected ``good" arms. However, the ``good" arms within a subset can in fact be ``bad" in terms of the overall arm set $[n]$. To see this, consider the following example instance: $m$ arms have a mean reward $\epsilon$, and the rest of $n-m$ arms have a mean reward $-\epsilon$. Then any $\epsilon$-approximate top-$m$ arms need to have mean $\epsilon$. Suppose the subset $G_1$ contains $m' < m$ arms with mean $\epsilon$ and some other arms with mean $-\epsilon$. 
Ideally we would like to apply median elimination to find those top-$m'$ arms with mean $\epsilon$ within $G_1$. However, parameter $m'$ is not known, and we will apply median elimination on $G_1$ by selecting some $l$ arms. If $l < m'$, then the returned $l$ arms will not include all the top-$m'$ arms in $G_1$, and therefore fail to identify the final top-$m$ arms. On the other hand, if $l > m'$, then $\max^{l}_{i \in G_1} \mu_i = -\epsilon$. Any arm in $G_1$ is ranked in the top-$l$ within $G_1$, and the problem is trivial to solve. The returned $l$ arms,  even though are top-$l$ within $G_1$, are not guaranteed to contain those top-$m'$ arms with mean reward $\epsilon$.

To successfully apply the divide-and-conquer approach, we need a ``blind'' algorithm that returns a subset containing the approximate top-$m'$ arms, ideally with certain graceful transition of the confidence values.

\begin{definition}
The algorithm $\textsf{A}(\epsilon, \delta, m, [n], \sigma_{1:n}^2)$ is said to satisfy
the $(\epsilon, \delta')$ top-$m'$ condition, where $m' \leq m$, if with probability at least $1 - \delta'$, $\max_{j \in R^{\textsf{A}}}^{m'} \mu_j \geq \max_{j \in [n]}^{m'} \mu_i - \epsilon$. 
\end{definition}
The condition is equivalent to the standard $(\epsilon, \delta)$ top-$m$ arm identification requirement, if $m' = m$ and $\delta' = \delta$.
We first restate the median elimination algorithm presented in Algorithm \ref{alg:ME} (the halving algorithm \citep{kalyanakrishnan2010efficient}), with the necessary changes on the constants and the variance values taken into account (note the input $2m$).

\begin{algorithm}[h]
    \caption{\textsf{MedElim}($\epsilon, \delta, 2m, [n], \sigma^2_{1:n}$) }\label{alg:ME}
    \begin{algorithmic}
    	\STATE Initialize $S_1 = [n]$, $\ell = 1$ and $\epsilon_\ell = (\epsilon/3) \frac{3^\ell}{4^{\ell}},~ \delta_\ell = \frac{\delta/4}{2^\ell}$\;
	\WHILE{$|S_\ell| > 2m$}
	\STATE Pull arm-$i$ $t_{i, \ell} = \frac{2 \sigma_i^2}{(\epsilon_\ell/2)^2} \ln\frac{m}{\delta_{\ell}}$ times and calculate their sample mean $\hat{\mu}_{i, \ell}$ for each $i \in S_\ell$\;
	\STATE Update the candidate set as $S_{\ell + 1} = \argmax_{i \in S_{\ell}}^{1: \max(\lfloor |S_{\ell}|/2 \rfloor, 2m)} \hat{\mu}_{i, \ell}$\;
	\STATE Let $\ell = \ell+1$\;
	\ENDWHILE
	\STATE \textbf{Return:} $S_{\ell}$ \;
    \end{algorithmic}
\end{algorithm}


The following lemma summarizes the sample complexity of the MedElim algorithm with the aforementioned transition in the confidence  values for $m'=1,2,\ldots,m$ for the $2m$ return arms. This algorithm will be used as a building block for the variance-grouped median elimination algorithm given next. The proof of this lemma can be found in the appendix.

\begin{lemma} \label{lem:me}
For any $\sigma^2_{1:n}$, if $\max_{i \in [n]} \sigma_i^2 /\min_{j \in [n]} \sigma_j^2 \leq 2$, the \textnormal{MedElim} algorithm has an expected stopping time
\begin{align}
O\left( \frac{\sum_{i \in [n]} \sigma_i^2}{\epsilon^2} \left( \ln \frac{1}{\delta} + \ln(m) \right) \right).
\end{align}
Moreover, for any $m' \leq m$, the MedElim algorithm satisfies the $(\epsilon, \frac{m'}{m} \delta)$ top-$m'$ condition. \end{lemma}

Now we are in a position to provide the proposed algorithm below, which we refer to as the variance-grouped median elimination algorithm.

\begin{algorithm}[h]
    \caption{\textsf{V-MedElim}($\epsilon, \delta, m, [n], \sigma^2_{1:n}$)}\label{alg:V-ME}
    \begin{algorithmic}
 	\STATE Partition $[n]$ into groups $G_1, \ldots, G_k$ by (\ref{def:G_j}) \;
	 \FOR{$j \in 1:k$}
 	\STATE $R_j = \textsf{MedElim}(\epsilon/2, \delta/2, 2m, G_j, \sigma^2_{G_j})$;
	\ENDFOR
	\STATE Let $G = \cup_{j = 1}^k R_j$;
	\STATE $R = \textsf{WNElim}(\epsilon/2, \delta/2, m, G, \sigma^2_{G})$; 
	\STATE \textbf{Return:} $ R$ 
    \end{algorithmic}
\end{algorithm}


%
%


The performance of proposed algorithm is summarized in the following theorem.

\begin{theorem}\label{thm:vme}
The variance-grouped median elimination algorithm solves the $(\epsilon, \delta)$ top-$m$ arm identification problem for any $\epsilon > 0$ and $0< \delta<1$, and the expected number of samples is
\begin{align}
O\left( \sum_{i \in [n]} \frac{\sigma_i^2}{\epsilon^2} \ln\frac{1}{\delta} + \sum_{i \in G^{m}} \frac{\sigma_i^2}{\epsilon^2} \ln(m) + \sum_{j \in G^{l}} \frac{\sigma_j^2}{\epsilon^2} \Ent(\sigma^2_{G^{r}}) \right).
\end{align}
\end{theorem}

\begin{proof}[Proof of Theorem \ref{thm:vme}]
Without loss of generality, assume $[m]$ is the set of top-$m$ arms. For any $j$ with $G_j \cap [m] \not= \emptyset$ and $i \in G_j \cap [m]$, arm-$i$ must be one of top-$|G_j \cap [m]|$ arms in $G_j$. Let $m'_j = |G_j \cap [m]|$ be the number of top-$m$ arms contained in $G_j$. By Lemma \ref{lem:me}, with probability at least $1 - \frac{m'_j}{m} \frac{\delta}{2}$, 
\begin{align}
\max{}^{m'_j}_{l \in R_j} \mu_{l} & \geq \max{}^{m'_j}_{l \in G_j \cap [m]} \mu_l - \epsilon/2 \notag \\
& \geq \max{}_{l \in [n]}^{m} \mu_l - \epsilon / 2. 
\end{align}
It implies that with probability at least $1 - \sum_{j = 1}^k \frac{m'_j}{m} \frac{\delta}{2} = 1 - \frac{\delta}{2}$, there are at least $\sum_{j = 1}^k m'_j = m$ arms in $G = \cup_{j = 1}^k R_j$ that are $\epsilon/2$-approximate top-$m$. In other words, event $\max^m_{l \in G } \mu_l \geq \max^m_{l \in [n]} \mu_l - \epsilon/2$ occurs with probability at least $1 - \frac{\delta}{2}$.

Conditioned on this event occurring, Lemma \ref{lem:union} implies that with probability at least $1- \frac{\delta}{2}$, the returned set $R$ of the weighted naive elimination over $G = \cup_{j=1}^k R_j$ satisfies
\begin{align}
\min_{l \in R} \mu_l \geq \max{}^m_{l \in G} \mu_l - \epsilon/2 \geq \max{}^m_{l \in [n]} \mu_l - \epsilon.
\end{align}
Thus with probability at least $1 - \delta$, all arms in $R$ are $\epsilon$-approximate top-$m$.

Recall the definition of $G^l, G^m, G^r$ in Section \ref{sec:main}. The total number of samples used in the median elimination subroutine is $O\left( \sum_{i \in G^m} \frac{\sigma_i^2}{\epsilon^2}\left(\ln\frac{1}{\delta} + \ln(m) \right)\right)$. The number of samples used in the weighted naive elimination subroutine is $O\left( \sum_{i \in [n]} \frac{\sigma_i^2}{\epsilon^2}\left(\ln\frac{1}{\delta} + \Ent(\sigma^2_{G^r}) \right) \right)$. By Lemma \ref{lem:red-up}, the expected total number of samples is $O\left( \sum_{i \in [n]} \frac{\sigma_i^2}{\epsilon^2} \ln\frac{1}{\delta} + \sum_{i \in G^{m}} \frac{\sigma_i^2}{\epsilon^2} \ln(m) \right.$ $ \left.+ \sum_{j \in G^{l}} \frac{\sigma_j^2}{\epsilon^2} \Ent(\sigma^2_{G^{r}}) \right)$.
\end{proof}

\paragraph{An illustrative example:}  In the following example, we show the number of required samples by the variance-grouped median elimination algorithm given in Theorem \ref{thm:vme} achieves an order-wise improvement over $\frac{\sum_{i \in [n]} \sigma^2_i}{\epsilon^2}(\ln(1/\delta) + \text{Ent}(\sigma^2_{1:n})) $ and $\frac{\sum_{i \in [n]} \sigma^2_i}{\epsilon^2}(\ln(1/\delta) + \ln(m))$. 
Take some integer $k \geq 2$ as an auxiliary parameter in this problem setting, and denote $\ell = \lceil \log(k) \rceil$. Let $\log(m) = k$ and $\log(n) = k^2$. We aim to approximately identify the top-$m$ arms out of $n$ arms. Among these $n$ arms, there are $2^{i}$ arms with the same variance $2^{-i}$ for each $i = 0, 1, \ldots, \ell-1$, and the rest $n - \sum_{i = 0}^{\ell-1} 2^{i} = 2^{k^2} - 2^{\ell} + 1$ arms have the same variance $2^{-k^2} \ell / k$. Then $G^m$ is the set of arms with variances $2^{-k^2}\ell/k$, and $G^l$ is the set of arms with variances $2^{-i}$ for $i=0, 1, \ldots, \ell-1$. It is seen that
\begin{align}
\sum_{j \in G^m} \sigma^2_j & = (2^{k^2} - 2^\ell + 1) 2^{-k^2} \ell /k = \Theta(\ell/k), \\
\sum_{j \in G^l} \sigma^2_j & = \sum_{i = 0}^{\ell -1} 2^i 2^{-i} = \ell = \Theta(\log(k)), 
\end{align}
which implies $\sum_{j \in [n]} \sigma^2_j = \Theta(\log(k))$. Furthermore, we can calculate that 
\begin{align}
\text{Ent}(\sigma^2_{G^r}) = \Theta(\Ent(\sigma^2_{G^l})) = \Theta(\log(k)).
\end{align}
Thus the number of required samples by the variance-grouped median elimination algorithm is of order 
\begin{align}
\Theta(\ln(k) \ln(1/\delta) + \ln(k)^2 / \epsilon^2). \label{eqn:eg-vme}
\end{align}
Since $\Ent(\sigma^2_{1:n}) = \Theta(k)$ and $\ln(m) = \Theta(k)$, it is seen that $\frac{\sum_{i \in [n]} \sigma^2_i}{\epsilon^2}(\ln(1/\delta) + \text{Ent}(\sigma^2_{1:n})) $ and $\frac{\sum_{i \in [n]} \sigma^2_i}{\epsilon^2}(\ln(1/\delta) + \ln(m))$ are of the same order 
\begin{align}
\Theta(\ln(k)\ln(1/\delta) + k\ln(k) / \epsilon^2). \label{eqn:eg-adapt}
\end{align}
The detailed calculation of the entropy values used above is given in the supplementary material. Fix $\delta > 0$ as constant, comparing the numbers of required samples in (\ref{eqn:eg-vme}) and (\ref{eqn:eg-adapt}), which are of order $\Theta(\ln(k)^2 / \epsilon^2)$ and $\Theta(k \ln(k) / \epsilon^2)$, respectively, it is seen that the variance-grouped median elimination algorithm provides an order-wise improvement in this example setting by reducing a factor $k$ to $\ln(k)$.

\emph{Remark.} Our result establishes the theoretical optimality of the proposed algorithm through a matching lower bound provided in the following section. However, the empirical performance of the proposed algorithm suffers from large multiplicative factors introduced by the Median Elimination subroutine. More aggressive elimination based algorithm, such as the algorithms proposed in \citep{hassidim2020optimal}, can be used as a subroutine to improve the multiplicative factor while maintaining the same order.

\section{The Lower Bound} \label{sec:low}

In the homogeneous variance setting, the previous lower bound \citep{kalyanakrishnan2012pac} on worst-case $(\epsilon, \delta)$-PAC top-$m$ identification leveraged the change-of-measure technique and was proved by contradiction. The approach leads to a large multiplicative factor and is also difficult to utilize in the heterogeneous variance case. The lower bound was later tightened and generalized to the instance-dependent case in \citep{chen2017nearly} and \citep{simchowitz2017simulator}. Their approach assumed that the algorithms have a uniform preference over the arms at the beginning, which is reasonable in the homogeneous setting but not in the heterogeneous setting.

We derive a flexible simple inequality to better take into account the heterogeneous variances, given in Lemma \ref{lem:two-point}. Applying this lemma, we formulate the lower bound as an optimization problem, whose dual formulation (Lemma \ref{lem:optimization}) is then studied. The eventual lower bound is given in the following theorem, obtained by considering several feasible solutions to the dual problem.

\begin{theorem} \label{thm:lower}
There exists some universal constant $c > 0$, that for any  $0 < \epsilon$, $0 < \delta < 0.1$, $m < n/2$, $\sigma_{1:n}^2$ and any valid algorithm, there exists an instance with the given variances such that the expected number of samples of the algorithm is at least 
\begin{align}
c\left( \sum_{i \in [n]} \frac{\sigma_i^2}{\epsilon^2} \ln\frac{1}{\delta} + \sum_{i \in G^{m}} \frac{\sigma_i^2}{\epsilon^2} \ln(m) + \sum_{j \in G^{l}} \frac{\sigma_j^2}{\epsilon^2} \Ent(\sigma^2_{G^{r}}) \right).
\end{align}
\end{theorem} 

\subsection{Dual Formulation of the Lower Bound}

We first introduce an inequality in the lemma below, which helps us connect the sample complexity with a multi-hypothesis testing problem.
\begin{lemma} \label{lem:two-point}
For any two probability measure $P, Q$ on the same measurable space $(\Omega, \Fc)$, if $\Ec \in \Fc$ with $P(\Ec) \geq 1 - \delta > Q(\Ec)$, we have
\begin{align}
Q(\Ec) \geq B(\delta) e^{- \frac{D(P || Q)}{1-\delta}},
\end{align}
where $D(\cdot || \cdot)$ is the Kullback-Leibler divergence and $B(\delta) = e^{- \frac{\Ent(\delta, 1-\delta)}{1- \delta}}$ is a strictly decreasing function with $B(0.1) > 0.69$. 
\end{lemma}

Fix any algorithm $\textsf{A}$ with inputs $(\epsilon, \delta, m, [n], \sigma^2_{1:n})$ that solves the $(\epsilon, \delta)$ top-$m$ arm identification problem. Consider the Gaussian instances where the $i$-th arm has a Gaussian distribution with variance $\sigma^2_i$. Denote $P_I$ as the probability measure induced by the learning process of applying algorithm $\textsf{A}$ on Gaussian bandit instance $I \in \Ic(\sigma^2_{1:n})$. 

Let $\epsilon' > \epsilon$ be some parameter that can be arbitrarily close to $\epsilon$. 
For any subset $M \subset [n]$ with $|M| = m$ and any index $l \in [n] \setminus M$, we first construct an instance $I_{l, M} \in \Ic(\sigma^2_{1:n})$ by specifying the reward means of each arm as follows: the $l$-th arm has mean $0$, the arms in $M$ have mean $\epsilon'$, and the rest have mean $-\epsilon'$. The only $\epsilon$-approximate top-$m$ arms of instance $I_{l, M}$ are clearly $M$. Similarly, for each subset $F \subset [n]$ with $|F| = m-1$ and any index $l \in [n] \setminus F$, we then construct an instance $I_{l, F} \in \Ic(\sigma^2_{1:n})$. In instance $I_{l, F}$, the $l$-th arm has mean $0$, the arms in $F$ have mean $\epsilon'$, and the rest arms have mean $-\epsilon'$. The only $\epsilon$-approximate top-$m$ arm set of instance $I_{l, F}$ is clearly $F \cup \{l\}$. These are the possible hypotheses we will consider.

Given an instance $I_{l, M}$, if $F = M \setminus \{i\}$ for some $i \in M$, it is clear that instances $I_{l, M}$ and $I_{l, F}$ differ only at the $i$-th arm. Denote $t_{l, F, i}$ as the expected number of pulls of the $i$-th arm by algorithm $\textsf{A}$ on instance $I_{l, F}$. The KL-divergence can be calculated as $D(P_{I_{l, F}} || P_{I_{l, M}}) = \frac{2 \epsilon'^2}{\sigma^2_i} t_{l, F, i}$; see Lemma 5.1 in \citep{lattimore2020bandit} for more details. Since $\textsf{A}$ solves the $(\epsilon, \delta)$ top-$m$ arm identification problem, we have $P_{I_{l, F}}( R^{\textsf{A}} = F \cup \{l\}) \geq 1 - \delta > \delta \geq P_{I_{l, M}}(R^{\textsf{A}} = F \cup \{l\})$. Applying Lemma \ref{lem:two-point} on $P_{I_{I, F}}$, $P_{I_{l, M}}$ and event $\{R^{\textsf{A}} = F \cup \{l\}\}$  gives
\begin{align}
P_{I_{l,M}}( R^{\textsf{A}} & = F \cup \{l\} ) \geq B(\delta) e^{-\frac{D(P_{I_{l, F}} || P_{I_{l, M}})}{1 - \delta}} \notag \\
& = B(\delta) e^{- \frac{2 \epsilon'^2}{\sigma_i^2}\frac{t_{l, F, i}}{1 - \delta}}. \label{eqn:app-two-point}
\end{align}
This inequality holds for any $F = M \setminus\{i\}$ with $i \in M$. In addition, events $\{ R^{\textsf{A}} = M \cup \{l\} \setminus \{i\} \}$'s are disjoint for any $i \in M \cup \{l\}$, and they are also disjoint with the event $\{ R^{\textsf{A}} = M\}$. It follows that $\sum_{i \in M} P_{I_{l, M}}\left( R^{\textsf{A}} = M \cup \{l\} \setminus \{i\} \right) \leq 1- P_{I_{l, M}}( R^{\textsf{A}} = M ) \leq \delta $. 
Summing inequality (\ref{eqn:app-two-point}) for all $i \in M$ gives
\begin{align}
& \delta \geq \sum_{i \in M} P_{I_{l, M} }(R^{\textsf{A}} = M \cup \{l\} \setminus \{i\}) \notag \\
& \geq \sum_{i \in M} B(\delta) \exp\left(- \frac{2 \epsilon'^2}{\sigma_i^2}\frac{t_{l, M\setminus \{i\}, i}}{1 - \delta} \right). \label{eqn:constraints}
\end{align}

In the worst-case, algorithm $\textsf{A}$ takes at least $\max_{F, l \notin F} \sum_{j \notin F \cup \{l\}} t_{l, F, j}$ samples in expectation. Any valid algorithm has to satisfy (\ref{eqn:constraints}), and thus the sample complexity $\textnormal{SC}(\epsilon, \delta, m, [n], \sigma_{1:n}^2)$ is lower bounded by the optimal value of the following optimization problem: 
\begin{align}
&\text{minimize:} \qquad \max_{F \subset [n]: |F| = m-1, ~l \not\in F} \sum_{j \notin F \cup \{l\}} t_{l, F, j}  \\
&\text{subject to:}\qquad \sum_{i \in M} \exp\left(- t_{l, M \setminus\{i\}, i}/\theta_i\right) \leq \delta', \notag \\
& \quad\qquad\qquad \quad \forall M \subset [n], |M| = m,~\forall l \notin M,
\end{align}
where $\theta_i = \frac{(1-\delta) \sigma_i^2}{2 \epsilon^2}, \forall i \in [n]$ and $\delta' = \frac{\delta}{B(\delta)}$. Though this problem is convex, it is difficult to solve explicitly. Therefore, we consider its (restricted) dual formulation in the following lemma.
\begin{lemma} \label{lem:optimization}
For $\epsilon > 0$, $\delta < 0.25$, $m < n/2$, $(\sigma_i^2)_{i \in [n]}$, $\textnormal{SC}(\epsilon, \delta, m, [n], \sigma_{1:n}^2) \geq \frac{1-\delta}{2\epsilon^2} v^*$, where $v^*$ is the optimal value of the following optimization problem:
\begin{align}
&\text{maximize:} \quad \sum_{M \subset [n]: |M| = m} \left(\sum_{l \in M} \eta_{M \setminus \{l\}} \sigma^2_l \right)  \times \notag \\
& \hspace{1.5cm} \left(\ln \frac{B(\delta)}{\delta} + \Ent(\{\eta_{M \setminus \{l\}} \sigma^2_l\}_{ l \in M} ) \right) \\
&\text{subject to:} \quad \sum_{F \subset [n]: |F| = m-1} \eta_F = 1, \notag \\ 
& \hspace{1.5cm} \quad \eta_F \geq 0, ~\forall F \subset [n], |F|=m-1.
\end{align}
\end{lemma} 

Though the dual formulation is still difficult to solve, by the weak duality, we can derive lower bounds for the primal problem by assigning specific feasible values to the dual variables $\eta_F$'s. In addition, each $\eta_F$ is a probability mass function and has a clear operational meaning, which is the worst-case prior distribution of the underlining instance being one of $\{ I_{l, F} \}_{l \notin F}$.

\subsection{Dichotomy of the lower bound}

As shown in Theorem \ref{thm:lower}, the lower bound of the sample complexity consists of three terms
\begin{align}
\underbrace{\sum_{i \in [n]} \frac{\sigma_i^2}{\epsilon^2} \ln\frac{1}{\delta}}_{\textnormal{I}} + \underbrace{\sum_{i \in G^{m}} \frac{\sigma_i^2}{\epsilon^2} \ln(m)}_{\textnormal{II}} + \underbrace{\sum_{j \in G^{l}} \frac{\sigma_j^2}{\epsilon^2} \Ent(\sigma^2_{G^{r}})}_{\textnormal{III}}.
\end{align}
We will discuss each term from the viewpoint of the dual formulation in Lemma \ref{lem:optimization}. The optimal value $v^*$ of the optimization in Lemma \ref{lem:optimization} can be lower bounded by the average of the objective function values $v_1, v_2, v_3$  when assigning the variables certain feasible values in the dual optimization problem, i.e., $v^* = \Omega(v_1 + v_2 + v_3 )$. We  construct three sets of feasible dual variables $\eta_F$'s, the resultant values $v_{1:3}$ will induce Term I-III, respectively.

It is straightforward to see that Term I can be obtained by assigning $\eta_F$'s uniformly, and thus we can focus on Term II and Term III. More precisely, we aim to lower bound the optimal value of the following optimization problem: 
\begin{align}
&\text{maximize:} \quad \sum_{M \subset [n]: |M| = m} \left(\sum_{l \in M} \eta_{M \setminus \{l\}} \sigma^2_l \right) \times\notag \\
& \hspace{2.2cm} \Ent(\{\eta_{M \setminus \{l\}} \sigma^2_l\}_{ l \in M} )     \label{eqn:opt} \\
&\text{subject to:} \quad \sum_{F \subset [n]: |F| = m-1} \eta_F = 1, \notag \\
& \hspace{1cm} \quad \eta_F \geq 0, ~\forall F \subset [n], |F|=m-1. 
\end{align}

Firstly, to study the sample complexity induced by $\sigma^2_{G^{m}}$, we specify a feasible assignment of dual variables $\eta_F$'s as follows. For any $F \subset G^{m}$ with $|F| = m-1$, let $\eta_{F} = \frac{\prod_{i \in F} \sigma^2_i}{\sum_{F' \subset G^{m}: |F'| = m-1} \prod_{j \in F} \sigma_i^2}$; and for any $F \not\subset G^{m}$ with $|F| = m-1$, set $\eta_{F} = 0$. Then $\Ent(\{\eta_{M \setminus \{l\}} \sigma^2_l\}_{ l \in M} ) = \ln(m)$ for any $M \subset G^{m}$ with $|M| = m$. 
Formally, Term II is the introduced by the following lemma.
\begin{lemma} \label{lem:Gmore}
The optimal value of the optimization (\ref{eqn:opt}) is lower-bounded by $\frac{1}{3}\sum_{j \in G^{m}} \sigma_j^2 \ln(m) $.
\end{lemma}

Secondly, to study the complexity induced by $\sigma^2_{G^{l}}$, we consider the reduced arm set $G^{r} \supset G^{l}$. Define $L \subset G^{r}$ with $|L| = 2m$ as the arms with $2m$ largest variances in $G^r$. We can verify that $\sum_{i \in L} \sigma^2_i$ dominates $\sum_{j \in G^{r}} \sigma^2_j$. 
Moreover, $\Ent(\sigma^2_{G^{r}})$ and $\Ent(\sigma^2_{L})$ behave similarly, and thus we can focus on the arms in $L$. Rigorously, the following lemma justifies this choice.

\begin{lemma} \label{lem:complexity-Gr}
Let $\eta_F = \binom{2m}{m-1}^{-1}$ for any $F \subset L$ with $|F| = m-1$ and $\eta_F = 0$ otherwise. 
The objective function of the optimization problem (\ref{eqn:opt}) is at least $c' \sum_{i \in G^l} \sigma^2_i \Ent(\sigma^2_{G^r}) - \ln(2)\sum_{i \in L} \sigma^2_i$, for some constant $c' > 0$. 
\end{lemma}
The first item in Lemma \ref{lem:complexity-Gr} is exactly Term III, and the second item $-\ln(2)\sum_{i \in G^{l}} \sigma^2_i$ can be absorbed into Term I.

\section{Conclusion}
\label{sec:conclusion}

We study the worst-case sample complexity of $(\epsilon, \delta)$ top-$m$ arm identification problem with heterogeneous reward variances. The heterogeneity of reward variances is measured by certain entropy-like function. We propose the variance-grouped median elimination algorithm, which combines the advantages of the median elimination algorithm and the weighted naive elimination algorithm in a divide-and-conquer manner. Matching lower bound of the worst-case sample complexity was devised using a dual formulation and finding suitable feasible solutions.

%

%
%


\clearpage
\appendix

\thispagestyle{empty}

\onecolumn \makesupplementtitle

\section{Proofs for Section \ref{sec:main}}
 
We will need the following well known inequality frequently. 
\begin{lemma}[Hoeffding's inequality] \label{lem:hoeffding}
Let $X_{1:n}$ be $n$ independent random variables follow some $\sigma^2$-sub-Gaussian distribution with mean $\mu$. Let $\hat{\mu}$ be their sample mean. Then the following inequalities hold
\begin{align}
\Pb\left( \hat{\mu} - \mu \geq \epsilon \right) \leq e^{ - \frac{\epsilon^2 n}{2 \sigma^2} }, \quad \Pb\left( \hat{\mu} - \mu \leq- \epsilon \right) \leq e^{ - \frac{\epsilon^2 n}{2 \sigma^2} }.
\end{align}
\end{lemma}

\begin{lemma}[Restate Lemma \ref{lem:red-up}]
For any $m \geq 2$, $ \Ent(\sigma^2_{G^{r}}) \leq 8 \ln(m)$.
\end{lemma}
\begin{proof}[Proof of Lemma \ref{lem:red-up}]
For any choice of $\sigma^2_{1:n}$. Let $s_j = \sum_{i \in G'_j} \sigma^2_i$ for each $i = 1, \ldots, k$. By the grouping property of entropy, we have
\begin{align}
\Ent(\sigma^2_{G^{r}}) & = \Ent( s_{1:k} ) + \sum_{j =1 }^k \frac{ s_{j} }{\sum_{i = 1}^k s_{i} } \Ent(\sigma^2_{G'_j}) \\
& \leq \Ent(s_{1:k}) + \ln(2m),
\end{align}
where the inequality is due to the principal of maximum entropy.

For $j = 1, \ldots, k$, if $|G'_j| > 0$, we have $ 2^{j-1} \leq s_j /\underline{\sigma}^2 < 2m 2^{j}$, otherwise $s_j = 0$. Without loss of generality, assume $\underline{\sigma}^2 = 1$ and $s_k > 0$. Let $s_{1:k}$ be the assignment with the largest entropy $\Ent(s_{1:k})$. 
If there are only $2m$ non-zero $s_{1:k}$, we have $\Ent(s_{1:k}) \leq \ln(2m)$ and the lemma is already proved. When there are more than $2m$ non-zero $s_{1:k}$, we have
\begin{align}
\sum_{j = 1}^{k-2m+1} s_j \leq 2m \sum_{j = 1}^{k-2m+1} 2^{j} = 4m(2^{k-2m+1} - 1) < 4m 2^{k- 2m + 1},
\end{align}
and $s_k \geq 2^{k-1}$. It follows that
\begin{align}
\sum_{j = 1}^{k - 2m + 1} s_j & = \frac{\sum_{j = 1}^{k - 2m + 1} s_j}{\sum_{i = k-2m+2}^{k} s_i + \sum_{j = 1}^{k - 2m + 1} s_j } \sum_{j = 1}^k s_j \\
& \leq \frac{\sum_{j=1}^{k-2m+1} s_j}{s_{k} + \sum_{j = 1}^{k-2m+1}s_j} \sum_{j = 1}^k s_j < \frac{4m 2^{k-2m+1} }{2^{k-1} + 4m 2^{k-2m+1} } \sum_{j = 1}^k s_j\\
& = \frac{4m 2^{-2m + 2}}{1  + 4m 2^{-2m + 2}} \sum_{j = 1}^k s_j.
\end{align}
We can then write
\begin{align}
\Ent(s_{1:k}) & = \Ent(\sum_{j = 1}^{k-2m+1}s_{j}, s_{k-2m+2:k}) + \frac{\sum_{j = 1}^{k - 2m + 1} s_j}{\sum_{j = 1}^k s_j} \Ent(s_{1:k-2m+1}) \\
& \leq \ln(2m) + \frac{4m 2^{-2m + 2}}{1  + 4m 2^{-2m + 2}} \Ent(s_{1:k}),
\end{align}
where the equality is by the grouping property of entropy function, and the inequality is by $\Ent(s_{1:k-2m+1}) \leq \Ent(s_{1:k})$ since $s_{1:k}$ is the optimal assignment in terms of the largest entropy with $k$ subsets, thus assignment $s_{1: k-2m+1}$ has smaller entropy. 
It implies $\Ent(s_{1:k}) \leq (1 + 4m 2^{-2m+2}) \ln(2m) \leq 3\ln(2m)$.
We thus have $\Ent(\sigma^2_{G^{r}}) \leq 4 \ln(2m) \leq 8 \ln(m)$. 
\end{proof}

\section{Proofs for Section \ref{sec:alg}}

\begin{lemma}[ Restate Lemma \ref{lem:union}]
Let $\omega_i = \delta \frac{\sigma_i^2}{\sum_{j=1}^n \sigma_j^2}$, the weighted naive elimination algorithm takes
\begin{align}
8 \sum_{i \in [n]}\frac{\sigma_i^2}{\epsilon^2} \left( \ln\frac{1}{\delta} + \Ent(\sigma_{1:n}^2)\right) \label{eqn:lem-union}
\end{align}
samples, and solves the $(\epsilon, \delta)$ top-$m$ arm identification problem for any $\epsilon > 0$ and $0 < \delta < 1$. 
\end{lemma}
\begin{proof}[Proof of Lemma \ref{lem:union}] The stopping time is clearly
\begin{align}
\sum_{i = 1}^n \frac{2 \sigma_i^2}{(\epsilon/2)^2}\ln\frac{1}{\omega_i} = 8\frac{\sum_{i=1}^n \sigma_i^2}{\epsilon^2} \left( \ln\frac{1}{\delta} + \Ent(\sigma_{1:n}^2) \right).
\end{align}

After the arms have been pulled and the reward observations collected, by Hoeffding's inequality (Lemma \ref{lem:hoeffding}), we have $\Pb(\hat{\mu}_i \leq \mu_i - \epsilon/2 ) \leq \omega_i$ for any $i \in [m]$ and $\Pb(\hat{\mu}_j \geq \mu_j + \epsilon/2) \leq \omega_j$ for any $j \in [n] \setminus [m]$. Since $\sum_{i \in [n]} \omega_j = \delta$, the union bound implies that the event $\Ec=\{\hat{\mu}_i > \mu_i - \epsilon/2,\forall i \in [m]\} \cap \{\hat{\mu}_j < \mu_j + \epsilon/2,\forall j \in [n]\setminus[m]\}$ occurs with probability at least $1 - \delta$.

Suppose event $\Ec$ occurs. Consider a threshold $\mu_m - \epsilon/2$. Firstly, for any $i \in [m]$, $\hat{\mu}_i > \mu_i - \epsilon/2 \geq \mu_m - \epsilon/2$. In addition, any $j \in [n] / [m]$ with $\hat{\mu}_j > \mu_m - \epsilon/2$ must satisfy $\mu_j + \epsilon/2 > \hat{\mu}_j > \mu_m - \epsilon/2$, which implies $\mu_j > \mu_m - \epsilon$, i.e., the $j$-th arm is $\epsilon$-approximate top-$m$. In other words, any arm with a sample mean greater than the threshold $\mu_{m} - \epsilon/2$ must be $\epsilon$-approximate top-$m$. Since there are at least $m$ arms with sample means greater than $\mu_{m} - \epsilon/2$, the $m$ selected arms must be $\epsilon$-approximate top-$m$. 
\end{proof}

\begin{lemma} [Restate Lemma \ref{lem:me}]
For any $\sigma^2_{1:n}$, if $\max_{i \in [n]} \sigma_i^2 /\min_{j \in [n]} \sigma_j^2 \leq 2$, the \textnormal{MedElim} algorithm has an expected stopping time
\begin{align}
O\left( \frac{\sum_{i \in [n]} \sigma_i^2}{\epsilon^2} \left( \ln \frac{1}{\delta} + \ln(m) \right) \right).
\end{align}
Moreover, for any $m' \leq m$, the MedElim algorithm satisfies the $(\epsilon, \frac{m'}{m} \delta)$ top-$m'$ condition.
\end{lemma}

\begin{proof}[Proof of Lemma \ref{lem:me}] We study the stopping time and accuracy separately. 

\noindent\textbf{Stopping time analysis:} Recall that $\overline{r} = \frac{\max_{i \in [n]} \sigma_i^2}{\min_{j \in [n]} \sigma_j^2}$. It is clear that the size of the candidate set $\Sc_{\ell}$ decreases as $|\Sc_{\ell}| \leq \frac{n}{2^{\ell - 1}}$. The sum of variances in the candidate set $\Sc_{\ell}$ decreases as follows 
\begin{align}
\frac{\sum_{i \in \Sc_\ell} \sigma_i^2}{\sum_{j \in [n]} \sigma^2_j} & \leq \frac{\sum_{i \in \Sc_\ell} \overline{r} \underline{\sigma}^2 }{ \sum_{j \in [n]} \underline{\sigma}^2 } \leq \overline{r}\frac{|\Sc_\ell|}{n} \leq \frac{\overline{r}}{2^{\ell-1}}.
\end{align}


This implies that
\begin{align}
\frac{\sum_{i \in S_\ell} \sigma_i^2}{(\epsilon_\ell/2)^2} = 36 \frac{16^\ell}{9^\ell}\frac{\sum_{i \in S_\ell} \sigma_i^2}{\epsilon^2} \leq 72 \overline{r} \frac{8^\ell}{9^\ell} \frac{\sum_{i=1}^n \sigma_i^2}{\epsilon^2}.
\end{align}
The (random) total number of samples is thus upper bounded by
\begin{align}
&\sum_{\ell = 1}^{\infty} \sum_{i \in S_\ell} t_{i, \ell} = \sum_{\ell = 1}^{\infty} \frac{2 \sum_{i \in S_\ell} \sigma_i^2}{(\epsilon_\ell/2)^2} \ln\left( \frac{m}{ \delta_\ell} \right) \\
&\leq \overline{r}\frac{144 \sum_{i = 1}^n \sigma_i^2}{\epsilon^2} \sum_{\ell = 1}^\infty  \frac{8^\ell}{9^\ell}\left(\ell \ln(2) + \ln \frac{4m}{\delta}\right)\\
& = O\left( \overline{r} \frac{\sum_{i=1}^n \sigma_i^2}{\epsilon^2} \left(\ln\frac{1}{\delta} + \ln(m)\right)\right),
\end{align}
with probability one. Thus the expected stopping time is of order $O\left( \tilde{r} \frac{\sum_{i \in [n]} \sigma^2_i}{\epsilon^2} \left(\ln\frac{1}{\delta} + \ln(m) \right)\right)$.

\textbf{Accuracy analysis.} Take an arbitrary $\ell \geq 1$ with $| \Sc_{\ell} | > 2m$. Fix some $m' \leq m$. Let $1_{\ell}, 2_{\ell}, \ldots, m'_{\ell}$ be the indices of the top-$m'$ arms in $S_{\ell}$ obtained in iteration-$(\ell-1)$. 
For any $i \in [m']$, by Hoeffding's inequality (Lemma \ref{lem:hoeffding}), we have $\Pb(\hat{\mu}_{i_{\ell}, \ell} > \mu_{i_\ell} - \epsilon_\ell/2) \geq 1 - \frac{1}{m}\delta_\ell$. Define the event $\Ec_{\ell} = \{\forall i \in [m'], ~ \hat{\mu}_{i_{\ell}, \ell} > \mu_{i_\ell} - \epsilon_\ell/2 \}$. By applying the union bound over $i \in [m']$, it is straightforward to verify that $\Pb(\Ec_\ell) \geq 1- \frac{m'}{m}\delta_\ell$.

Conditioned on event $\Ec_\ell$ occurring, consider a threshold $\mu_{m'_{\ell}} - \epsilon_\ell/2$. It is clear that for any $i \in [m']$, $\hat{\mu}_{i_\ell, \ell} > \mu_{i_\ell} - \epsilon/2 \geq \mu_{m'_\ell} - \epsilon/2$. Thus any arm in $\{1_\ell, \ldots, m'_\ell\}$ has an empirical mean greater than the threshold $\mu_{m'_{\ell}} - \epsilon_\ell/2$. In iteration-$\ell$, $|\Sc_{\ell + 1}|$ arms with the largest empirical means are selected from set $\Sc_{\ell}$. 
\begin{itemize}
\item If the selected arm with the smallest sample mean $\min\{\hat{\mu}_{i, \ell} :~i \in \Sc_{\ell+1} \}$ is less than or equal to the threshold, then all the arms in $\{1_\ell, \ldots, m'_\ell \}$ must be selected and they are still the top-$m'$ arms within $\Sc_{\ell+1}$. It implies that $\mu_{m'_{\ell + 1}} = \mu_{m'_{\ell}} > \mu_{m'_{\ell}} - \epsilon_\ell$.
\item On the other hand, if the selected arm with the smallest sample mean is greater than the threshold, some arms in $\{1_\ell, \ldots, m'_{\ell}\}$ may not be selected. Define the set of bad arms $B_{\ell} := \{ i \in \Sc_\ell:~ \mu_{i} < \mu_{m'_\ell} - \epsilon_\ell \}$. A bad arm will be selected only if its empirical mean is greater than the threshold. Denote the set of bad arms with such overestimated sample means as $N_{m', \ell} = \{j \in B_\ell:~ \hat{\mu}_{j, \ell} > \mu_{m'_\ell} - \epsilon_\ell/2 \}$. Then there are at most $|N_{m', \ell}|$ bad arms in $\Sc_{\ell + 1}$. 
If $|N_{m', \ell}| \leq |\Sc_{\ell + 1}| - m'$, at least $m'$ good arms remain in $\Sc_{\ell+1}$, which guarantees $\mu_{m'_{\ell + 1}} \geq \mu_{m'_{\ell}} - \epsilon_\ell$.
\end{itemize}
These two situations indicate that conditioned on $\Ec_\ell$, $|N_{m', \ell}| \leq |\Sc_{\ell + 1}| - m'$ implies $\mu_{m'_{\ell + 1}} \geq \mu_{m'_{\ell}} - \epsilon_\ell$. It follows that
\begin{align*}
\Pb\left( \mu_{m'_{\ell + 1}} < \mu_{m'_{\ell}} - \epsilon_\ell | \Ec_\ell \right) & \leq \Pb\left( |N_{m', \ell}| \geq |S_{\ell+1}| - m' + 1 | \Ec_{\ell} \right) \\
& \leq \frac{\Eb[ |N_{m', \ell}| | \Ec_{\ell}]}{ |S_{\ell+1}| - i + 1}.
\end{align*}
where the second inequality is due to Markov inequality. The expectation can be bounded by
\begin{align*}
\Eb[|N_{m', \ell}| | \Ec_\ell] & = \sum_{j \in B_\ell} \Pb\left( \hat{\mu}_{j, \ell} > \mu_{m'_{\ell}} - \epsilon_\ell/2 | \Ec_\ell \right) \\
& = \sum_{j \in B_\ell} \Pb\left( \hat{\mu}_{j, \ell} > \mu_{m'_{\ell}} - \epsilon_\ell/2 \right) \\
& \leq \sum_{j \in B_\ell} \Pb\left( \hat{\mu}_{j, \ell} > \mu_{j} + \epsilon_\ell/2 \right) \\
& \leq (|S_\ell| - m') \frac{\delta_\ell}{m},
\end{align*}
where the equality is because $\Ec_\ell$ is defined by the samples of arms in $[1_\ell, \ldots, m'_\ell]$ which are independent from the samples of arms in $B_\ell$, the first inequality is by $\mu_{m'_\ell} > \mu_{j}$ for $j \in B_\ell$, and the last inequality is by applying Hoeffding's inequality to each $\hat{\mu}_{j, l}, j \in B_\ell$ 
and $|B_\ell| \leq |\Sc_\ell| - m'$. We thus have
\begin{align*}
\Pb\left( \mu_{m'_{\ell + 1}} < \mu_{m'_{\ell}} - \epsilon_\ell | \Ec_\ell \right) & \leq \frac{\delta_{\ell}}{m} \frac{|S_{\ell}| - m' }{|S_{\ell+1}| - m' + 1} \\
& \leq \frac{\delta_{\ell}}{m} \frac{|\Sc_{\ell}| - m}{|\Sc_{\ell+1}| - m + 1} \quad \quad\quad\quad \text{by } m' \leq m\\
& \leq \frac{\delta_{\ell}}{m} \frac{2|\Sc_{\ell+1}| + 1 - m}{|\Sc_{\ell+1}| - m + 1} \quad \quad\quad~~ \text{by } |\Sc_\ell| \leq 2|\Sc_{\ell+1}| + 1 \\
& = \frac{\delta_{\ell}}{m} \left(2 +  \frac{m - 1}{|\Sc_{\ell+1}| - m + 1} \right) \\
& \leq \frac{\delta_\ell}{m} \left(2 +  \frac{m - 1}{2m- m + 1} \right) \quad\quad \text{by } |\Sc_{\ell+1}| \geq 2m\\
& < \frac{3 \delta_\ell}{m}.
\end{align*}

It follows that
\begin{align*}
\Pb\left( \mu_{m'_{\ell + 1}} < \mu_{m'_{\ell}} - \epsilon_\ell \right) & = \Pb(\Ec) \Pb\left( \mu_{m'_{\ell + 1}} < \mu_{m'_{\ell}} - \epsilon_\ell | \Ec \right) + \Pb(\Ec^c)\Pb\left( \mu_{m'_{\ell + 1}} < \mu_{m'_{\ell}} - \epsilon_\ell | \Ec^c \right) \notag \\
& \leq \Pb\left( \mu_{m'_{\ell + 1}} < \mu_{m'_{\ell}} - \epsilon_\ell | \Ec \right) + \Pb(\Ec^c) \\
& \leq \frac{3 \delta_\ell}{m} + \frac{m' \delta_\ell}{m} \leq \frac{4 m'}{m} \delta_\ell.
\end{align*}

The argument above holds for any $\ell \geq 1$ with $| S_{\ell} | > 2m$. The parameters satisfy
\begin{align*}
\sum_{\ell = 1}^\infty \epsilon_\ell = \frac{\epsilon}{3} \sum_{\ell = 1}^{\infty} (3/4)^{\ell} = \epsilon, \quad\quad
\sum_{\ell = 1}^\infty 4\delta_\ell = \delta \sum_{\ell = 1}^{\infty} (1/2)^{\ell} = \delta.
\end{align*}
The returned arm set is $R = \Sc_{\ell^*}$ for certain $\ell^*$, and thus with probability at least $1 - \frac{m'}{m} \delta$, the final returned arm set $R$ satisfies
\begin{align*}
\max{}^{m'}_{i \in R} \mu_i & = \max{}^{m'}_{i \in \Sc_{\ell^*}} \mu_i \\
& \geq \max{}^{m'}_{i \in \Sc_{\ell^* - 1}} \mu_i - \epsilon_{\ell^* - 1} \\
& \geq \cdots \\
& \geq \max{}^{m'}_{i \in \Sc_1} \mu_i - \sum_{\ell = 1}^{\ell^* - 1} \epsilon_{\ell} \\
& > \max{}^{m'}_{i \in [n]} \mu_i - \epsilon.
\end{align*}
The proof is thus complete.
\end{proof}

\paragraph{Calculation in the illustrative example} 
Recall the illustrative example, where $\log(m) = k$ and $\log(n) = k^2$ for some integer $k \geq 2$ and $\ell = \lceil \log(k) \rceil$. Among these $n$ arms, there are $2^{i}$ arms with the same variance $2^{-i}$ for each $i = 0, 1, \ldots, \ell-1$, and the rest $n - \sum_{i = 0}^{\ell-1} 2^{i} = 2^{k^2} - 2^{\ell} + 1$ arms have the same variance $2^{-k^2} \ell / k$. Then $G^m$ is the set of arms with variances $2^{-k^2}\ell/k$, and $G^l$ is the set of arms with variances $2^{-i}$ for $i=0, 1, \ldots, \ell-1$. It is seen that
\begin{align}
\sum_{j \in G^m} \sigma^2_j & = (2^{k^2} - 2^\ell + 1) 2^{-k^2} \ell /k = \Theta(\ell/k), \\
\sum_{j \in G^l} \sigma^2_j & = \sum_{i = 0}^{\ell -1} 2^i 2^{-i} = \ell = \Theta(\log(k)), 
\end{align}
which implies $\sum_{j \in [n]} \sigma^2_j = \Theta(\log(k))$. Furthermore, we can calculate that 
\begin{align}
\text{Ent}(\sigma^2_{G^l}) = \sum_{i = 0}^{\ell-1} \frac{2^i 2^{-i}}{\ell} \ln(2^i) =  \frac{\ln(2)}{2}(\ell-1) = \Theta(\ell) = \Theta(\log(k)).
\end{align}
Furthermore, we can calculate that 
\begin{align}
\sum_{j \in G^r} \sigma^2_j = 2m 2^{-k^2} \ell /k + \sum_{j \in G^l} \sigma^2_j = 2^{-k^2 + 1} \ell + \ell = \Theta(\ell) = \Theta(\log(k)).
\end{align}
Then the entropy values can be calculated as
\begin{align}
\text{Ent}(\sigma^2_{G^r}) & = \frac{\sum_{j \in G^r \slash G^l} \sigma^2_j}{\sum_{j \in G^r} \sigma^2_j} \text{Ent}(\sigma^2_{G^r \slash G^l}) + \frac{\sum_{j \in G^l} \sigma^2_j }{\sum_{j \in G^r} \sigma^2_j} \text{Ent}(\sigma^2_{G^l}) \\
&= \frac{2^{-k^2 + 1} \ell}{\sum_{j \in G^r} \sigma^2_j} \ln(2m) + \frac{\ell}{\sum_{j \in G^r} \sigma^2_j} \text{Ent}(\sigma^2_{G^l}) \\
&= \Theta\left( 2^{-k^2} k + \text{Ent}(\sigma^2_{G^l}) \right) \\
& = \Theta(\text{Ent}(\sigma^2_{G^l})) = \Theta(\log(k)),
\end{align}
and $\text{Ent}(\sigma^2_{G^m}) = \Theta(k^2)$ implies
\begin{align}
\text{Ent}(\sigma^2_{1:n}) & = \frac{\sum_{j \in G^m} \sigma^2_j}{\sum_{j \in [n]} \sigma^2_j} \text{Ent}(\sigma^2_{G^m}) + \frac{\sum_{j \in G^l} \sigma^2_j }{\sum_{j \in [n]} \sigma^2_j} \text{Ent}(\sigma^2_{G^l}) \\
&= \Theta\left( \frac{\ell/k}{\log(k)} k^2 +  \log(k) \right) = \Theta(k).
\end{align}

\section{A More Adaptive Median Elimination Algorithm}

Let us sort $\sigma^2_{1:n}$ in decreasing order, and denote the sorted variances as $\tilde{\sigma}^2_{1:n}$. For each $\ell \geq 1$, define $h_\ell := \max\{ j \geq m: \sum_{i \in [j]} \tilde{\sigma}^2_i \leq \frac{1}{2^{\ell-1}} \sum_{i \in [n]} \sigma^2_i \}$ if the set is not empty, otherwise $h_{\ell} = m$. Let $\ell^* := \min\{\ell \geq 1: h_{\ell} = m\}$. 

Define a ratio 
\begin{align}
\underline{r} := \min_{j \in [\ell^*-1]} \frac{h_{j+1}}{h_{j}}
\end{align}

\begin{algorithm}[h]
    \caption{\textsf{Adapted-MedElim}($\sigma^2_{1:n}, m, [n], \epsilon, \delta$) }\label{alg:ratio-me}
    \begin{algorithmic}
    	\STATE sInitialize $S_1 = [n]$, $\ell = 1$ and $\epsilon_\ell = (\epsilon/3) \frac{3^\ell}{4^{\ell}},~ \delta_\ell = \frac{\underline{r}\delta}{2^\ell}$\;
	\FOR{$\ell = 1, 2, \ldots, \ell^* - 1$}
	\STATE Pull arm-$i$ $t_{i, \ell} = \frac{2 \sigma_i^2}{(\epsilon_\ell/2)^2} \ln\frac{m}{\delta_{\ell}}$ times and calculate their sample mean $\hat{\mu}_{i, \ell}$ for each $i \in S_\ell$\;
	\STATE Update candidate set $S_{\ell + 1} = \argmax_{i \in S_{\ell}}^{1: h_{\ell+1}} \hat{\mu}_{i, \ell}$\;
	\ENDFOR
	\STATE \textbf{Return:} $S_{\ell^*}$ \;
    \end{algorithmic}
\end{algorithm}

In the homogeneous setting, the MedElim algorithm halves the complexity of the problem if the candidate set is halved. However, it should be noted that in the heterogeneous setting, simply halving the candidate set may not be efficient since the complexity would depend on the sum of the variances, instead of the number of the candidate arms. We can instead aim to half the sum of the variances of the candidate set. This discrepancy is less pronounced when the heterogeneity is low, and thus the MedElim algorithm performs reasonably well in such cases.


\begin{lemma} \label{lem:ratio-me}
The algorithm is valid and has an expected stopping time
\begin{align}
O\left( \sum_{i \in [n]}\frac{\sigma_i^2}{\epsilon^2} \left( \ln \frac{1}{\delta} + \ln(m) + \ln \frac{1}{\underline{r}} \right) \right).
\end{align}
\end{lemma}

\begin{proof}[Proof of Lemma \ref{lem:ratio-me}] We study the stopping time and accuracy separately. 

\noindent\textbf{Stopping time analysis:} First, notice the sum of variances in the candidate set decreases as follows:
\begin{align}
\sum_{i \in \Sc_\ell} \sigma_i^2 = \frac{\sum_{i \in \Sc_\ell} \sigma_i^2}{\sum_{i \in [n]} \sigma^2_i } \sum_{i \in [n]} \sigma_i^2 \leq \frac{\sum_{i \in [h_{\ell}]} \tilde{\sigma}_i^2}{\sum_{i \in [n]} \sigma^2_i } \sum_{i \in [n]} \sigma_i^2 \leq \frac{1}{2^{\ell-1}} \sum_{i \in [n]}^n \sigma_i^2.
\end{align}

This implies that
\begin{align}
\frac{\sum_{i \in S_\ell} \sigma_i^2}{(\epsilon_\ell/2)^2} = 36 \frac{16^\ell}{9^\ell}\frac{\sum_{i \in S_\ell} \sigma_i^2}{\epsilon^2} \leq 72 \overline{r} \frac{8^\ell}{9^\ell} \frac{\sum_{i=1}^n \sigma_i^2}{\epsilon^2}.
\end{align}
The stopping time is thus upper bounded by
\begin{align}
&\sum_{\ell = 1}^{\infty} \sum_{i \in S_\ell} t_{i, \ell} = \sum_{\ell = 1}^{\infty} \frac{2 \sum_{i \in S_\ell} \sigma_i^2}{(\epsilon_\ell/2)^2} \ln\left( \frac{m}{ \delta_\ell} \right) \\
&\leq \frac{144 \sum_{i = 1}^n \sigma_i^2}{\epsilon^2} \sum_{\ell = 1}^\infty  \frac{8^\ell}{9^\ell}\left(\ell \ln(2) + \ln \frac{m}{\delta} + \ln\frac{1}{\underline{r}}\right)\\
& = O\left( \frac{\sum_{i=1}^n \sigma_i^2}{\epsilon^2} \left(\ln\frac{1}{\delta} + \ln(m) + \ln\frac{1}{\underline{r}} \right)\right).
\end{align}
The expected stopping time is of order $O\left( \frac{\sum_{i=1}^n \sigma_i^2}{\epsilon^2} \left(\ln\frac{1}{\delta} + \ln(m) + \ln\frac{1}{\underline{r}} \right) \right)$.

\textbf{Accuracy analysis.} Take an arbitrary $\ell \in [\ell^* - 1]$, and it is clear that $|\Sc_{\ell}| = h_{\ell} > m$. Fix some $m' \leq m$. Let $1_{\ell}, 2_{\ell}, \ldots, m'_{\ell}$ be the indices of the top-$m'$ arms in $S_{\ell}$, respectively. 
For any $i \in [m']$, by Hoeffding's inequality (Lemma \ref{lem:hoeffding}), we have $\Pb(\hat{\mu}_{i_{\ell}, \ell} > \mu_{i_\ell} - \epsilon_\ell/2) \geq 1 - \frac{1}{m}\delta_\ell$. Define the event $\Ec_{\ell} = \{\forall i \in [m'], ~ \hat{\mu}_{i_{\ell}, \ell} > \mu_{i_\ell} - \epsilon_\ell/2 \}$. By applying the union bound over $i \in [m']$, it is straightforward to verify that $\Pb(\Ec_\ell) \geq 1- \frac{m'}{m}\delta_\ell$.

Conditioned on the event $\Ec_\ell$ occurring, consider a threshold $\mu_{m'_{\ell}} - \epsilon_\ell/2$. It is clear that for any $i \in [m']$, $\hat{\mu}_{i_\ell, \ell} > \mu_{i_\ell} - \epsilon/2 \geq \mu_{m'_\ell} - \epsilon/2$. Thus any arm in $\{1_\ell, \ldots, m'_\ell\}$ has empirical mean greater than the threshold $\mu_{m'_{\ell}} - \epsilon_\ell/2$. $|\Sc_{\ell + 1}| = h_{\ell+1}$ arms with the largest sample means are selected from set $\Sc_{\ell}$. 
\begin{itemize}
\item If the smallest selected sample mean $\min\{\hat{\mu}_{i, \ell} :~i \in \Sc_{\ell+1} \}$ is less or equal to the threshold, all arms in $\{1_\ell, \ldots, m'_\ell \}$ must be selected and they are still top-$m'$ arms within $\Sc_{\ell+1}$. It implies that $\mu_{m'_{\ell + 1}} = \mu_{m'_{\ell}} > \mu_{m'_{\ell}} - \epsilon_\ell$.
\item On the other hand, if the smallest selected sample mean is greater than the threshold, some arms in $\{1_\ell, \ldots, m'_{\ell}\}$ may not be selected. Define the set of bad arms $B_{\ell} := \{ i \in \Sc_\ell:~ \mu_{i} < \mu_{m'_\ell} - \epsilon_\ell \}$. A bad arm can be selected only if its empirical mean is greater than the threshold. Define the set of such overestimated bad arms as $N_{m', \ell} = \{j \in B_\ell:~ \hat{\mu}_{j, \ell} > \mu_{m'_\ell} - \epsilon_\ell/2 \}$. Then there are at most $|N_{m', \ell}|$ bad arms in $\Sc_{\ell + 1}$. 
If $|N_{m', \ell}| \leq |\Sc_{\ell + 1}| - m'$, at least $m'$ good arms remain in $\Sc_{\ell+1}$, which guarantees $\mu_{m'_{\ell + 1}} \geq \mu_{m'_{\ell}} - \epsilon_\ell$.
\end{itemize}
These two situations indicate that $|N_{m', \ell}| \leq |\Sc_{\ell + 1}| - m'$ implies $\mu_{m'_{\ell + 1}} \geq \mu_{m'_{\ell}} - \epsilon_\ell$ conditioned on $\Ec_\ell$. It follows that
\begin{align*}
\Pb\left( \mu_{m'_{\ell + 1}} < \mu_{m'_{\ell}} - \epsilon_\ell | \Ec_\ell \right) & \leq \Pb\left( |N_{m', \ell}| \geq |S_{\ell+1}| - m' + 1 | \Ec_{\ell} \right) \\
& \leq \frac{\Eb[ |N_{m', \ell}| | \Ec_{\ell}]}{ |S_{\ell+1}| - i + 1}.
\end{align*}
where the second inequality is by Markov inequality. The expectation can be bounded by
\begin{align*}
\Eb[|N_{m', \ell}| | \Ec_\ell] = \sum_{j \in B_\ell} \Pb\left( \hat{\mu}_{j, \ell} > \mu_{m'_{\ell}} - \epsilon_\ell/2 | \Ec_\ell \right) \leq (|S_\ell| - m') \frac{\delta_\ell}{m},
\end{align*}
where the inequality is by Hoeffding's inequality and $|B_\ell| \leq |\Sc_\ell| - m'$. We thus have
\begin{align*}
\Pb\left( \mu_{m'_{\ell + 1}} < \mu_{m'_{\ell}} - \epsilon_\ell | \Ec_\ell \right) & \leq \frac{\delta_{\ell}}{m} \frac{|S_{\ell}| - m' }{|S_{\ell+1}| - m' + 1} \\
& = \frac{\delta_{\ell}}{m} \frac{h_{\ell} - m'}{h_{\ell +1} - m' + 1} \\
& \leq \frac{\delta_{\ell}}{m} \frac{h_{\ell+1}/\underline{r} - m'}{h_{\ell+1} - m' + 1} \quad \quad\quad~~ \text{by } h_{\ell} \leq \frac{1}{\underline{r}} h_{\ell+1} \\
& = \frac{\delta_{\ell}}{m} \left(\frac{1}{\underline{r}} +  \frac{(1/\underline{r} - 1)m' - 1/\underline{r}}{h_{\ell+1} - m' + 1} \right) \\
& \leq \frac{\delta_\ell}{m} \left(1/\underline{r} + (1/\underline{r} - 1)m' - 1/\underline{r} \right) \quad\quad \text{by } h_{\ell+1} \geq m \geq m'\\
& = \frac{m' \delta_\ell}{m} (1/\underline{r} - 1).
\end{align*}

It follows that
\begin{align*}
\Pb\left( \mu_{m'_{\ell + 1}} < \mu_{m'_{\ell}} - \epsilon_\ell \right) & = \Pb(\Ec) \Pb\left( \mu_{m'_{\ell + 1}} < \mu_{m'_{\ell}} - \epsilon_\ell | \Ec \right) + \Pb(\Ec^c)\Pb\left( \mu_{m'_{\ell + 1}} < \mu_{m'_{\ell}} - \epsilon_\ell | \Ec^c \right) \notag \\
& \leq \Pb\left( \mu_{m'_{\ell + 1}} < \mu_{m'_{\ell}} - \epsilon_\ell | \Ec \right) + \Pb(\Ec^c) \\
& \leq \frac{m' \delta_\ell}{m}(1/\underline{r} - 1) + \frac{m' \delta_\ell}{m} =  \frac{1}{\underline{r}}\frac{m'}{m} \delta_\ell.
\end{align*}

The argument above holds for any $\ell \geq 1$ with $| S_{\ell} | > 2m$. The parameters satisfy
\begin{align*}
\sum_{\ell = 1}^\infty \epsilon_\ell = \frac{\epsilon}{3} \sum_{\ell = 1}^{\infty} (3/4)^{\ell} = \epsilon, \quad\quad
\sum_{\ell = 1}^\infty \frac{1}{\underline{r}}\delta_\ell = \delta \sum_{\ell = 1}^{\infty} (1/2)^{\ell} = \delta.
\end{align*}
The returned arm set $R = \Sc_{\ell^*}$ for some $\ell^*$. With probability at least $1 - \frac{m'}{m} \delta$, the final returned arm set $R$ satisfies
\begin{align*}
\max{}^{m'}_{i \in R} \mu_i & = \max{}^{m'}_{i \in \Sc_{\ell^*}} \mu_i \\
& \geq \max{}^{m'}_{i \in \Sc_{\ell^* - 1}} \mu_i - \epsilon_{\ell^* - 1} \\
& \geq \cdots \\
& \geq \max{}^{m'}_{i \in \Sc_1} \mu_i - \sum_{\ell = 1}^{\ell^* - 1} \epsilon_{\ell} \\
& > \max{}^{m'}_{i \in [n]} \mu_i - \epsilon.
\end{align*}
\end{proof}

\section{Proofs for Section \ref{sec:low}}

Define $\Ic(\sigma^2_{1:n}) := \{(\mu_{1:n}, \sigma^2_{1:n}): ~ \mu_{1:n} \in \Rb^n\}$. When $\sigma^2_{1:n}$ is obvious in the context, we simply write $\Ic(\sigma^2_{1:n})$ as $\Ic$. 
The sample complexity of the approximate top-$m$ identification problem under algorithm inputs $(\epsilon, \delta, m, [n], \sigma_{1:n}^2)$ is 
\begin{align}
\textnormal{SC}(\epsilon, \delta, m, [n], \sigma_{1:n}^2) := \inf_{\textsf{A}} \sup_{I \in \Ic(\sigma^2_{1:n})} \Eb_{I}[ T^{\textsf{A}} ],
\end{align}
where the infimum is taken over all valid algorithms, the supreme is taken over the instance class $\Ic(\sigma^2_{1:n}) := \{(\mu_{1:n}, \sigma^2_{1:n}): ~ \mu_{1:n} \in \Rb^n\}$, and the subscript $I$ in the expectation $\Eb_{I}[\cdot]$ indicates that it is with respect to bandit model $I$.

\begin{lemma} [Restate Lemma \ref{lem:two-point}]
For any two probability measure $P, Q$ on the same measurable space $(\Omega, \Fc)$, if $\Ec \in \Fc$ with $P(\Ec) \geq 1 - \delta > Q(\Ec)$, we have
\begin{align}
Q(\Ec) \geq B(\delta) e^{- \frac{D(P || Q)}{1-\delta}},
\end{align}
where $D(\cdot || \cdot)$ is the Kullback-Leibler divergence and $B(\delta) = e^{- \frac{\Ent(\delta, 1-\delta)}{1- \delta}}$ is a strictly decreasing function with $B(0.1) > 0.69$.
\end{lemma}
\begin{proof}[Proof of Lemma \ref{lem:two-point}]
Let $D_b(p, q) = p \ln \frac{p}{q} + (1-p) \ln \frac{1- p}{1-q}$ be the binary KL-divergence. Since $P(\Ec) \geq 1- \delta$, by the data processing inequality for the KL-divergence, we have
\begin{align}
& D(P||Q) \geq D_b(P(\Ec), Q(\Ec)) \geq D_b(1-\delta, Q(\Ec)) \\
& > (1 - \delta) \ln\frac{1 - \delta}{Q(\Ec)} + \delta \ln \delta \geq (1 - \delta) \ln \frac{B(\delta)}{Q(\Ec)},
\end{align}
where the second inequality is due to  $P(\Ec) \geq 1 - \delta > Q(\Ec)$, and the fact that $D_b(p, q)$ is monotonically increasing in $p$ in the range $[q,1]$ for any fixed $q$. We thus concludes that
\begin{align}
Q(\Ec) \geq B(\delta) e^{-\frac{D(P||Q)}{1-\delta}}.
\end{align}
\end{proof}

\begin{lemma} [Restate Lemma \ref{lem:optimization}]
For $\epsilon > 0$, $\delta < 0.25$, $m < n/2$, $(\sigma_i^2)_{i \in [n]}$, $\textnormal{SC}(\epsilon, \delta, m, [n], \sigma_{1:n}^2) \geq \frac{1-\delta}{2\epsilon^2} v^*$, where $v^*$ is the optimal value of the following optimization problem:
\begin{align}
&\text{maximize:} \quad \sum_{M \subset [n]: |M| = m} \left(\sum_{l \in M} \eta_{M \setminus \{l\}} \sigma^2_l \right) \left(\ln \frac{B(\delta)}{\delta} + \Ent(\{\eta_{M \setminus \{l\}} \sigma^2_l\}_{ l \in M} ) \right) \\
&\text{subject to:} \quad \sum_{F \subset [n]: |F| = m-1} \eta_F = 1, \quad \eta_F \geq 0, ~\forall F \subset [n], |F|=m-1.
\end{align}
\end{lemma} 

\begin{proof}[Proof of Lemma \ref{lem:optimization}]
We have shown in Section \ref{sec:low} that $\textnormal{SC}(\epsilon, \delta, m, [n], \sigma_{1:n}^2)$ is lower bounded by the optimal value of the following optimization problem: 
\begin{align}
&\text{minimize:} \qquad \max_{F \subset [n]: |F| = m-1, ~l \not\in F} \sum_{j \notin F \cup \{l\}} t_{l, F, j}  \\
&\text{subject to:}\qquad \sum_{i \in M} \exp\left(- t_{l, M \setminus\{i\}, i}/\theta_i\right) \leq \delta', \quad \forall M \subset [n], |M| = m,~\forall l \notin M,
\end{align}
where $\theta_i = \frac{(1-\delta) \sigma_i^2}{2 \epsilon^2}, \forall i \in [n]$ and $\delta' = \frac{\delta}{B(\delta)}$. This problem is equivalent to the following convex optimization.
\begin{align}
&\min_{t, \tau} \qquad \tau\\
&\text{s.t.} \qquad \sum_{j \notin F \cup \{l\}} t_{l, F, j} \leq \tau, \quad \forall F \subset [n] \setminus \{l\}: |F| = m-1, \forall l \in [n]  \\
&\qquad \sum_{i \in M} \exp\left(- t_{l, M\setminus\{i\}, i}/\theta_i\right) \leq \delta', \quad \forall M \subset [n] \setminus \{l\}: |M| = m, \forall l \in [n].  
\end{align}
For simplicity, we use notation $\sum_{l, F}$ and $\sum_{l, M}$ to indicate $\sum_{l \in [n]} \sum_{F \subset [n] \setminus \{i\}: |F| = m-1}$ and $\sum_{l \in [n]} \sum_{M \subset [n] \setminus \{i\} : |M| = m}$, respectively. The Lagrangian of the optimization problem above is
\begin{align}
L(t, \tau, \eta, \lambda) = \tau + \sum_{l, F} \eta_{l, F} \left( \sum_{j \notin F \cup \{l\}} t_{l, F, j} - \tau \right) + \sum_{l, M} \lambda_{l, M} \left( \sum_{i \in M}\exp\left(- t_{l, M\setminus\{i\}, i}/\theta_i\right) - \delta' \right)
\end{align}
It is straightforward to check the optimization problem satisfies Slater's condition by assigning large enough $t_{l, F, j}$ and $\tau$ values. Since the optimization problem is convex, the optimal value equals to $\sup_{\eta, \lambda} \inf_{t, \tau}  L(t, \tau, \eta, \lambda)$ according to the strong duality. For the saddle point, we must have $\sum_{l, F} \eta_{l, F} = 1$, or else $\inf_{t, \tau} L(t, \tau, \eta, \lambda) = -\infty$. Decision variable $\tau$ can thus be omitted. Let $L(t, \eta, \lambda) = L(t, \tau, \eta, \lambda)$ by restricting $\sum_{l, F} \eta_{l, F} = 1$. The derivative can be calculated that
\begin{align}
\frac{\mathrm{d}L(t, \eta, \lambda)}{\mathrm{d} t_{l, F, i}} = \eta_{l, F} - \frac{\lambda_{l, F \cup\{i\} }}{\theta_i} \exp(-t_{l, F, i} / \theta_i).
\end{align}
It implies that when $\eta_{l, F} > 0$ and $\lambda_{l, F \cup\{i\}} > 0$, $t_{l, F, i} = \theta_i \ln \frac{ \lambda_{l, F \cup\{i\}} }{\eta_{l, F} \theta_i}$. Define $\ln(0) = -\infty$ and let $0 \cdot \infty = 0$. The extended real valued function $g(\eta, \lambda)$ for $\sum_{l, F} \eta_{l, F} = 1$, $\eta_{l, F} \geq 0$ and $\lambda_{l, M} \geq 0$, is
\begin{align}
g(\eta, \lambda) := \inf_{t} L(t, \eta, \lambda) &= \sum_{l, F} \eta_{l, F} \sum_{i \notin F \cup \{l\}} \theta_i \ln \lambda_{l, F \cup\{i\}} - \sum_{l, F} \eta_{l, F} \sum_{i \notin F \cup \{l\}} \theta_i\ln(\eta_{l, F} \theta_i) \notag \\ 
&\quad + \sum_{l, M} \left( \sum_{i \in M} \eta_{l, M \setminus \{i\}} \theta_i- \delta'\lambda_{l, M}  \right).
\end{align}

This dual function has two set of variables, however one of them can be eliminated explicitly as follows. For fixed $\eta$'s with $\sum_{l, F} \eta_{l, F} = 1$ and $\eta_{l, F} \geq 0$, the function is separable with respect to $\lambda$'s, and thus we can maximize $g(\eta, \lambda)$ by optimizing each individual $\lambda_{l, M}$ separately. It is straightforward to verify that $\lambda_{l, M} = \left(\sum_{F, i : F \cup \{i\} = M} \eta_{l, F} \theta_i \right)/\delta'$.

Since $\eta$'s, $\theta$'s and $\delta'$ are positive, the assignments of $\lambda$'s are also positive, which satisfy the constraints in the dual program. Plug it into $g(\eta, \lambda)$, we have the induced objective as
\begin{align}
g(\eta) & = \sum_{l, F} \eta_{l, F} \sum_{i \notin F \cup \{l\}} \theta_i \ln \frac{ \sum_{F', i' : F' \cup \{i'\} = F \cup \{i\}} \eta_{l, F} \theta_i }{\eta_{l, F} \theta_i \delta'} \\
& = \sum_{F} \sum_{i \notin F} \sum_{l \notin F \cup \{i\}} \eta_{l, F} \theta_i \ln \frac{ \sum_{F', i' : F' \cup \{i'\} = F \cup \{i\}} \eta_{l, F} \theta_i }{\eta_{l, F} \theta_i \delta'}.
\end{align}
and the dual variables $\eta$'s lie in a probability simplex. 

Further constraining the problem by requiring $\eta_{F} := (n - m) \eta_{l, F}$ for all $l \notin F$ reduces the number of dual variables, but does not change the fact that any valid assignment of $\eta_{F}$'s will provide a lower bound to the original primal problem. The following restricted objective will be considered:
\begin{align}
g(\eta) & = \sum_{F} \sum_{i \notin F} \sum_{l \notin F \cup \{i\}} \frac{\eta_{F}}{n-m} \theta_i \ln \frac{ \sum_{F', i' : F' \cup \{i'\} = F \cup \{i\}} \eta_{F} \theta_i }{\eta_{F} \theta_i \delta'} \\
& = \sum_{F} \sum_{i \notin F} \eta_{F}\theta_i \ln \frac{ \sum_{F', i' : F' \cup \{i'\} = F \cup \{i\}} \eta_{F} \theta_i }{\eta_{F} \theta_i \delta'}.
\end{align}
The optimal value of the optimization above is lower bounded by
\begin{align}
&\text{maximize} \qquad \sum_{M \subset [n], |M| = m } \left( \left(\sum_{j \in M} \eta_{M \setminus \{j\}} \theta_j \right) \left( \Ent(\{\eta_{M \setminus \{j\}} \sigma^2_j\}_{ j \in M } )  + \ln \frac{B(\delta)}{\delta} \right)  \right) \\
&\text{subject to} \qquad \sum_{F \subset [n]: |F| = m-1} \eta_F = 1, \quad \eta_F \geq 0, ~\forall F \subset [n], |F|=m-1.
\end{align}
The lemma is proved.
\end{proof}

Recall that the optimization in (\ref{eqn:opt}) is
\begin{align}
&\text{maximize:} \quad \sum_{M \subset [n]: |M| = m} \left(\sum_{l \in M} \eta_{M \setminus \{l\}} \sigma^2_l \right) \Ent(\{\eta_{M \setminus \{l\}} \sigma^2_l\}_{ l \in M} )     \label{eqn:opt-recall} \\
&\text{subject to:} \quad \sum_{F \subset [n]: |F| = m-1} \eta_F = 1, \quad \eta_F \geq 0, ~\forall F \subset [n], |F|=m-1. 
\end{align}
\begin{lemma} [Restate Lemma \ref{lem:Gmore}]
The optimal value of the optimization (\ref{eqn:opt-recall}) is lower-bounded by $\frac{1}{3}\sum_{j \in G^{m}} \sigma_j^2 \ln(m) $.
\end{lemma}
\begin{proof}[Proof of Lemma \ref{lem:Gmore}]
The objective function of equation (\ref{eqn:opt-recall}) can be written as
\begin{align}
&\sum_{F \subset [n]: |F| = m-1} \sum_{i \notin F} \eta_F \sigma_i^2 \ln\left( \frac{\sum_{F' \cup \{j\} =F \cup \{i\}} \eta_{F'} \sigma_j^2}{\eta_F \sigma_i} \right) \notag\\
&\qquad= \sum_{i \in [n]} \sum_{F: i \not\in F} \eta_F \sigma_i^2 \ln\left( \frac{\sum_{F' \cup \{j\} =F \cup \{i\}} \eta_{F'} \sigma_j^2}{\eta_F \sigma_i} \right) \label{eqn:sum-F}
\end{align}
For any $F \subset G^{m}$ with $|F| = m-1$, let $\eta_{F} = \frac{\prod_{i \in F} \sigma^2_i}{\sum_{F' \subset G^{m}: |F'| = m-1} \prod_{j \in F} \sigma_i^2}$; and for any $F \not\subset G^{m}$ with $|F| = m-1$, set $\eta_{F} = 0$. In the following analysis $F$ indicates subset of $G^m$ with $|F| = m-1$ and $E$ indicates subset of $G^m$ with $|E| = m-2$. Items in (\ref{eqn:sum-F}) can be lower bounded as follows.
\begin{align}
& \quad \ln(m) \sum_{i \in G^m} \sigma^2_i \sum_{F: i \not\in F} \eta_F \\
& = \ln(m) \sum_{i \in G^{m}} \sigma^2_i \frac{ \sum_{F: i \notin F} \prod_{l \in F} \sigma^2_l}{\sum_{F: i \in F} \prod_{l \in F} \sigma^2_l + \sum_{F: i \not\in F} \prod_{l \in F} \sigma^2_l } \\
& = \ln(m) \sum_{i} \sigma^2_i \left(\frac{\sum_{F: i \in F} \prod_{l \in F} \sigma^2_l}{\sum_{F: i \notin F} \prod_{l \in F} \sigma^2_l} + 1 \right)^{-1} \\
& =  \ln(m) \sum_{i \in G^{m}} \sigma^2_i  \left((m-1)\frac{\sum_{F: i \in F} \prod_{l \in F} \sigma^2_l}{(m-1)\sum_{F: i \notin F} \prod_{l \in F} \sigma^2_l} + 1 \right)^{-1} \\
& = \ln(m) \sum_{i \in G^{m}} \sigma^2_i  \left((m-1)\sigma_i^2 \frac{\sum_{E:i \notin E} \prod_{l \in E} \sigma^2_l}{\sum_{E:i \notin E} \prod_{l \in E} \sigma^2_l \left(  \sum_{j \in G^m \setminus E \setminus \{i\}} \sigma_j^2 \right)} + 1 \right)^{-1} \\
& \geq \ln(m) \sum_{i \in G^{m}} \sigma^2_i  \left((m-1)\sigma_i^2 \frac{\sum_{E} \prod_{l \in E} \sigma^2_l}{\sum_{E} \prod_{l \in E} \sigma^2_l \left(  \min_{F: i \in F} \sum_{j \in G^m \setminus F} \sigma_j^2 \right)} + 1 \right)^{-1} \\
& =  \ln(m) \sum_{i \in G^{m}} \sigma^2_i \left(\frac{(m-1)\sigma^2_i}{\sum_{j \in G^m} \sigma^2_j - \max_{F: i \in F} \sum_{l \in F} \sigma^2_l } + 1 \right)^{-1},
\end{align}
where the last inequality is by $\sum_{j \in G^m \setminus E \setminus \{i\}} \sigma_j^2 \geq  \min_{F: i \in F} \sum_{j \in G^m \setminus F} \sigma_j^2$ for any $\Ec$. 
Recall the definition of $G^m$: there are $G_{1:k}$ groups partitioning $[n]$ and $G^{m} = \cup_{j: |G_j| > 2m} G_j$. Consider the group $G_{k'} \subset G^m$ with the largest index $k' \leq k$. 
Since the heterogeneity within group $G_{k'}$ is at most $2$, we have $\max_{ i \in G^m} \sigma^2_i \leq 2 \sigma^2_j$ for any $j \in G_{k'}$.
Then for any $F \subset G^m$ and any $i \in G^m$,
\begin{align}
\frac{(m-1)\sigma^2_i}{\sum_{j \in G^m} \sigma^2_j - \sum_{l \in F} \sigma^2_l } & =  \frac{(m-1)\sigma^2_i}{\sum_{j \in G^m \setminus F} \sigma^2_j }  \leq \frac{(m-1) \sigma^2_i}{\sum_{j \in G_{k'} \setminus F} \sigma^2_j} \leq \frac{2 (m-1)}{|G_{k'} \setminus F|} \leq \frac{2(m-1)}{m+1} < 2,
\end{align}
where the first inequality is by $G_{k'} \subset G^m$, the second inequality is by $\sigma_i^2 \leq 2 \sigma^2_j$ for any $j \in G_{k'}$, and the third inequality is by $|G_{k'}| > 2m$. It follows that
\begin{align}
& \ln(m) \sum_{i} \sigma^2_i \left(\frac{(m-1)\sigma^2_i}{\sum_{j \in G^m} \sigma^2_j - \max_{F: i \in F} \sum_{l \in F} \sigma^2_l } + 1 \right)^{-1} \\
& \geq \ln(m) \sum_{i} \sigma^2_i \left( 2 + 1\right)^{-1} = \frac{1}{3} \ln(m) \sum_{j} \sigma_j^2.
\end{align}
\end{proof}

\begin{lemma} \label{lem:ent-red-L}
There exists some constant $0 < c' < 1$, that for any choices of $\sigma^2_{1:n}$, 
$\Ent(\sigma^2_{L}) \geq c' \Ent(\sigma^2_{G^{r}}) - c' \ln(2)$.
\end{lemma}
\begin{proof}[Proof of Lemma \ref{lem:ent-red-L}]
By the grouping property of entropy, we have
\begin{align}
\Ent(\sigma^2_{G^{r}}) & = \Ent(\sum_{j \in L} \sigma^2_j, \sum_{i \in G^{r} \setminus L} \sigma^2_i) \\
&~+ \frac{ \sum_{i \in L} \sigma^2_i }{\sum_{j \in G^{r}} \sigma^2_j } \Ent(\sigma^2_{L}) + \left(1 -  \frac{ \sum_{i \in L} \sigma^2_i }{\sum_{j \in G^{r}} \sigma^2_j }\right) \Ent(\sigma^2_{G^{r} \setminus L}) \\
& < \ln(2) + \Ent(\sigma^2_{L}) + \left(1 -  \frac{ \sum_{i \in L} \sigma^2_i }{\sum_{j \in G^{r}} \sigma^2_j }\right) 8 \ln(m) \\
&\leq \ln(2) + 33 \Ent(\sigma^2_{L}),
\end{align}
where the first inequality is due to the principal of maximum entropy and Lemma \ref{lem:red-up}, and the last inequality is by Lemma \ref{lem:L-low}.
\end{proof}

\begin{lemma}[Retate Lemma \ref{lem:complexity-Gr}]
Let $\eta_F = \binom{2m}{m-1}^{-1}$ for any $F \subset L$ with $|F| = m-1$ and $\eta_F = 0$ otherwise. 
There exists some constant $c' > 0.005$. The objective of optimization (\ref{eqn:opt}) is at least $c' \sum_{i \in G^l} \sigma^2_i \Ent(\sigma^2_{G^r}) - \ln(2)\sum_{i \in L} \sigma^2_i $.
\end{lemma}

\begin{proof} Recall $L \subset G^{r}$ with $|L| = 2m$ is the subset of arms with largest variances within $G^r$. For any $M \subset L$ with $|M| = m$, by the grouping property of entropy function we have
\begin{align}
\Ent(\sigma^2_L) & = \Ent(\sum_{i \in M} \sigma^2_i , \sum_{j \in L \setminus M} \sigma^2_j) + \frac{\sum_{i \in M} \sigma^2_i }{\sum_{j \in L} \sigma^2_j} \Ent(\sigma^2_{M}) + \frac{\sum_{i \in L\setminus M} \sigma^2_i }{\sum_{j \in L} \sigma^2_j} \Ent(\sigma^2_{L\setminus M}) \\
& \leq \ln(2) + \frac{\sum_{i \in M} \sigma^2_i }{\sum_{j \in L} \sigma^2_j} \Ent(\sigma^2_{M}) + \frac{\sum_{i \in L\setminus M} \sigma^2_i }{\sum_{j \in L} \sigma^2_j} \Ent(\sigma^2_{L\setminus M}),
\end{align}
where the inequality is by the principal of maximum entropy. Multiply $\sum_{j \in L} \sigma^2_j$ on both side, and we have
\begin{align}
\sum_{i \in M} \sigma_j^2 \Ent(\sigma^2_M) + \sum_{i \in L\setminus M} \sigma^2_i \Ent(\sigma^2_{L \setminus M}) \geq \sum_{j \in L} \sigma^2_j (\Ent(\sigma^2_l) - \ln(2)).
\end{align}
Since $|M| =|L\setminus M| = m$, summing the inequality above for each $M \subset L$ with $|M| = m$ and multiplying by $\frac{1}{2 \binom{2m}{m-1}}$ gives us
\begin{align}
&  \sum_{M \subset L : |M| = m} \frac{1}{\binom{2m}{m-1}} \sum_{i \in M} \sigma^2_i \Ent(\sigma^2_M)  \geq \frac{\binom{2m}{m}}{ 2 \binom{2m}{m-1} } (\Ent(\sigma^2_L) - \ln(2)) \sum_{i \in L} \sigma^2_i \\
& \geq \frac{1}{2} (\Ent(\sigma^2_L) - \ln(2)) \sum_{i \in L} \sigma^2_i = \frac{1}{2} \Ent(\sigma^2_L) \sum_{i \in L} \sigma^2_i - \frac{\ln(2)}{2} \sum_{i \in L} \sigma^2_i \\
& \geq \frac{1}{2} \sum_{i \in L} \sigma^2_i \frac{\Ent(\sigma^2_{G^r}) - \ln(2) }{33} - \frac{\ln(2)}{2} \sum_{i \in L} \sigma^2_i\\
& \geq \frac{1}{6} \sum_{i \in G^{r}} \sigma^2_i \frac{ \Ent(\sigma^2_{G^r}) }{33} - \frac{1}{2}\sum_{i \in L} \sigma^2_i \frac{\ln(2)}{33} - \frac{\ln(2)}{2} \sum_{i \in L} \sigma^2_i \\
& \geq \frac{1}{174} \sum_{i \in G^r} \sigma^2_i \Ent(\sigma^2_{G^r}) - \ln(2) \sum_{i \in L} \sigma^2_i,
\end{align}
where the second inequality is by $\frac{\binom{2m}{m}}{\binom{2m}{m-1} } \geq 1$, the third and forth inequalities are by Lemma \ref{lem:ent-red-L}.
\end{proof}


\section{Supporting Lemmas}
\begin{lemma}[Lemma 5.1 in \citep{lattimore2020bandit}] \label{lem:kl-chain}
Given two bandit instances $I = (\mu_{1:n}, \sigma^2_{1:n} )$ and $I' = (\mu'_{1:n}, \sigma'^2_{1:n})$, and let $P_I$ and $P_{I'}$ be the probability measure associated with the bandit instances, respectively. Then for any algorithm $\textsf{A}$ with the number of pulling for each arm-$i$ as $T_i^{\textsf{A}}$, which is a random variable, let $\tau^{\textsf{A}}$ be the bandit process and let $\Pb_{I, \pi}$ and $\Pb_{I', \pi}$ be the probability measures induced by $\tau^{\textsf{A}}$ on instance $I$ and $I'$, respectively. We have
\begin{align}
D(\Pb_{I, \textsf{A}} || \Pb_{I', \textsf{A}}) = \sum_{i = 1}^n \Eb_I[T_i^{\textsf{A}}] D\left( \Nc(\mu_i, \sigma_i^2) || \Nc(\mu'_i, \sigma'^2_i)\right).
\end{align}
\end{lemma}

\begin{lemma}\label{lem:L-low}
For any $\sigma^2_{1:n}$, we have $\frac{ \sum_{i \in L} \sigma^2_i }{\sum_{j \in G^{r}} \sigma^2_j } \geq \frac{1}{3}$. In addition, 
\begin{align}
\left( 1 -  \frac{ \sum_{i \in L} \sigma^2_i }{\sum_{j \in G^{r}} \sigma^2_j }  \right) \ln(m) \leq 4 \Ent(\sigma^2_{L}),
\end{align}
for some constant $c > 0$.
\end{lemma}
\begin{proof}[Proof of Lemma \ref{lem:L-low}]
Suppose the minimum variance in $\sigma^2_{L}$ is $\tilde{\sigma}^2$. Let $\alpha = \frac{2m \tilde{\sigma}^2 }{\sum_{i \in L} \sigma^2_i} \in (0, 1]$, which implies $\sum_{i \in L} \sigma^2_i = 2m\tilde{\sigma}^2/\alpha$. In addition, $\sum_{j \in G^{r} \setminus L} \sigma^2_j \leq 2m \tilde{\sigma}^2 \sum_{i = 0}^{\infty} 2^{-i} = 4m \tilde{\sigma}^2$.
It is straightforward to verify that
\begin{align}
\frac{ \sum_{i \in L} \sigma^2_i }{\sum_{j \in G^{r}} \sigma^2_j } =  \frac{2m \tilde{\sigma}^2 / \alpha}{ \sum_{j \in G^{r} \setminus L} \sigma^2_j +  2m \tilde{\sigma}^2 / \alpha} \geq \frac{2m \tilde{\sigma}^2/\alpha}{4m\tilde{\sigma}^2 + 2m \tilde{\sigma}^2/\alpha} = \frac{1/\alpha}{2 + 1/\alpha} \geq \frac{1}{3},
\end{align}
which proves the first statement. It follows that
\begin{align}
1 -  \frac{ \sum_{i \in L} \sigma^2_i }{\sum_{j \in G^{r}} \sigma^2_j } \leq 1 - \frac{1/\alpha}{2 + 1/\alpha} = \frac{2}{2 + 1/\alpha} < \frac{2}{1 + 1/\alpha}.
\end{align}
By concavity of entropy function, $\Ent(\sigma^2_{L}) \geq \Ent\left( 1- \frac{2m-1}{2m}\alpha , \frac{\alpha}{2m}, \frac{\alpha}{2m}, \cdots, \frac{\alpha}{2m}\right)$. It implies that 
\begin{align}
& (1 + 1/\alpha) \Ent(\sigma^2_{L}) \\
& \geq (1 + 1/\alpha)\left(-(1 - \frac{2m-1}{m}\alpha) \ln\left(1 - \frac{2m-1}{2m}\alpha \right) + \frac{2m-1}{2m}\alpha \ln\frac{2m}{\alpha}  \right) \\
& \geq \frac{2m-1}{2m} \ln(2m) + \frac{2m-1}{2m} \ln\frac{1}{\alpha} - \left(\frac{1}{\alpha} - \frac{2m-1}{2m}\right) \ln\left( 1 - \frac{2m-1}{2m}\alpha \right) \\
& \geq \frac{1}{2} \ln(2m) - \frac{1}{2}\ln(\alpha) - (1/\alpha - 1) \ln(1 - \alpha)\\
& \geq \frac{1}{2} \ln(2m) > \frac{1}{2} \ln(m).
\end{align}
We thus have
\begin{align}
4 \Ent(\sigma^2_L) > \frac{2}{1 + 1/\alpha} \ln(m) > \left( 1 -  \frac{ \sum_{i \in L} \sigma^2_i }{\sum_{j \in G^{r}} \sigma^2_j }  \right) \ln(m).
\end{align}
\end{proof}

\end{document}